\documentclass{article}

\usepackage[accepted]{icml2020}

\usepackage{booktabs}
\usepackage[hidelinks]{hyperref}
\usepackage{color}
\usepackage{latexsym}              % symbols
\usepackage{amsmath}               % great math stuff
\usepackage{amssymb}               % great math symbols
\usepackage{amsfonts}              % for blackboard bold, etc
\usepackage{amsthm}                % for theorems, http://tex.stackexchange.com/a/130655
\usepackage{accents}
\usepackage{mathtools}
\usepackage{multicol}
\usepackage{multirow}
\usepackage{tikz}                  % to make drawing with TikZ
\usetikzlibrary{arrows,positioning,shapes}
\usepackage{pifont}% http://ctan.org/pkg/pifont
\usepackage{enumitem}
\usepackage{afterpage}
\usetikzlibrary{matrix}

\usepackage[mathcal]{eucal}

\usepackage{cleveref}
\crefname{assumption}{Assumption}{Assumptions}
\crefname{equation}{Eq.}{Eqs.}
\crefname{figure}{Figure}{Figures}
\crefname{table}{Table}{Tables}
\crefname{section}{Section}{Sections}
\crefname{appendix}{Appendix}{Appendices}
\crefname{algorithm}{Algorithm}{Algorithms}
\crefname{theorem}{Theorem}{Theorems}
\crefname{definition}{Definition}{Definitions}
\crefname{lemma}{Lemma}{Lemmas}
\crefname{proposition}{Proposition}{Propositions}
\crefname{corollary}{Corollary}{Corollaries}
\crefname{example}{Example}{Examples}
\crefname{appendix}{Appendix}{Appendixes}
\crefname{remark}{Remark}{Remark}

\newcounter{remark}[section]

\newcommand{\bigzero}{{\normalfont\Large 0}}

\newcommand{\calX}{\mathcal{X}}
\newcommand{\calY}{\mathcal{Y}}

\newcommand{\calZ}{\mathcal{Z}}
\newcommand{\calQ}{\mathcal{Q}}
\newcommand{\calE}{\mathcal{E}}
\newcommand{\calD}{\mathcal{D}}
\newcommand{\calP}{\mathcal{P}}
\newcommand{\calH}{\mathcal{H}}
\newcommand{\calO}{\mathcal{O}}

\newcommand{\calK}{\mathcal{K}}

\newcommand{\calG}{\mathcal{G}}
\newcommand{\calM}{\mathcal{M}}
\newcommand{\calC}{\mathcal{C}}
\newcommand{\calS}{\mathcal{S}}
\newcommand{\calR}{\mathcal{R}}
\newcommand{\calN}{\mathcal{N}}

\newcommand{\ERL}{\calE}
\newcommand{\ERS}{\calR}
\newcommand{\hull}{\operatorname{hull}}

\DeclareMathAlphabet{\mathsfsl}{OT1}{cmss}{m}{sl}

\renewcommand{\phi}{\varphi}

\newcommand{\Rspace}[1]{\mathbb{R}^{#1}}

\newcommand{\R}{\mathbb{R}}

\newcommand{\fstar}{f^\star}
\newcommand{\gstar}{g^\star}

\newcommand*{\defeq}{\mathrel{\vcenter{\baselineskip0.5ex \lineskiplimit0pt
                     \hbox{\scriptsize.}\hbox{\scriptsize.}}}%
                     =}

\newcommand{\argmin}{\operatorname*{arg\; min}}
\newcommand{\argmax}{\operatorname*{arg\; max}}

\newcommand{\Expect}{\operatorname{\mathbb{E}}}

\theoremstyle{plain}  % Plain style for theorem, defn, lemma, proposition, corollary
\newtheorem{theorem}{Theorem}[section]
\newtheorem{definition}[theorem]{Definition}

\newtheorem{lemma}[theorem]{Lemma}
\newtheorem{proposition}[theorem]{Proposition}

\newtheorem{corollary}[theorem]{Corollary}

\newtheorem{example}[theorem]{Example}

\icmltitlerunning{Consistent Structured Prediction with Max-Min Margin Markov Networks}

\begin{document}

\twocolumn[
\icmltitle{Consistent Structured Prediction with Max-Min Margin Markov Networks}

\icmlsetsymbol{equal}{*}

\begin{icmlauthorlist}
\icmlauthor{Alex Nowak-Vila}{to}
\icmlauthor{Francis Bach}{to}
\icmlauthor{Alessandro Rudi}{to}
\end{icmlauthorlist}

\icmlaffiliation{to}{INRIA - D\'epartement d’Informatique de l'Ecole Normale Sup\'erieure,
PSL Research University}

\icmlcorrespondingauthor{Alex Nowak-Vila}{alex.nowak-vila@inria.fr}

\icmlkeywords{Machine Learning, ICML}

\vskip 0.3in
]

\printAffiliationsAndNotice{}

\begin{abstract}
Max-margin methods for binary classification such as the support vector machine (SVM) have been extended to the structured prediction setting under the name of max-margin Markov networks ($\operatorname{M^3N}$), or more generally structural SVMs. Unfortunately, these methods are statistically inconsistent when the relationship between inputs and labels is far from deterministic.
We overcome such limitations by defining the learning problem in terms of a ``max-min'' margin formulation, naming the resulting method max-min margin Markov networks ($\operatorname{M^4N}$). We prove consistency and finite sample generalization bounds for $\operatorname{M^4N}$ and provide an explicit algorithm to compute the estimator. The algorithm achieves a generalization error of $O(1/\sqrt{n})$ for a total cost of~$O(n)$ projection-oracle calls (which have at most the same cost as the max-oracle from $\operatorname{M^3N}$). Experiments on multi-class classification, ordinal regression, sequence prediction and ranking demonstrate the effectiveness of the proposed method.
\end{abstract}

\section{Introduction} \label{sec:introduction}
Many classification tasks in machine learning lie beyond the classical binary and multi-class classification settings. In those tasks, the output elements are structured objects made of interdependent parts, such as sequences in natural language processing \cite{smith2011linguistic}, images in computer vision \cite{nowozin2011structured}, permutations in ranking or matching problems \cite{caetano2009learning} to name just a few \cite{bakir2007predicting}.
The structured prediction setting has two key properties that makes it radically different from multi-class classification, namely, the exponential growth of the size of the output space with the number of its parts, and the cost-sensitive nature of the learning task, as prediction mistakes are not equally costly. In sequence prediction, for instance, the number of possible outputs grows exponentially with the length of the sequences, and predicting a sequence with one incorrect character is better than predicting the whole sequence wrong.

Classical approaches in binary classification such as the \textit{non-smooth} support vector machine (SVM), and the \textit{smooth} logistic and quadratic plug-in classifiers have been extended to the structured setting under the name of max-margin Markov networks ($\operatorname{M^3N}$) \cite{taskar2004max} (or more generally structural SVM (SSVM) \cite{tsochantaridis2005large}), conditional random fields (CRFs) \cite{Lafferty:2001:CRF:645530.655813} and quadratic surrogate (QS) \cite{ciliberto2016consistent, ciliberto2018localized}, respectively. Theoretical properties of CRF and QS are well-understood. In particular, it is possible to obtain finite-sample generalization bounds of the resulting estimator on the cost-sensitive structured loss \cite{nowak2019general}. Unfortunately, these guarantees are not satisfied by $\operatorname{M^3N}$s even though the method is based on an upper bound of the loss. More precisely, it is known that the upper bound can be not tight (and lead to inconsistent estimation) when the relationship between input and output labels is far from deterministic \cite{liu2007fisher}, which it is essentially always the case in structured prediction due to the exponentially large output space. This means that the estimator does not converge to the minimizer of the problem leading to inconsistency.

Recently, a line of work \cite{fathony2016adversarial, fathony2018consistent, fathony2018efficient, fathony2018distributionally} proposed a consistent method based on an adversarial game formulation on the structured problem. However, their analysis does not allow to get generalization bounds and their proposed algorithm is specific for every setting with at least a complexity of $O(n^2)$ to obtain optimal statistical error when learning from $n$ samples. In this paper, we derive this method in the generic structured output setting from first principles coming from the binary SVM. We name this method max-min margin Markov networks ($\operatorname{M^4N}$), as it is based on a correction of the max-margin of $\operatorname{M^3N}$ to a `max-min' margin.
The proposed algorithm has essentially the same complexity as state-of-the-art methods for $\operatorname{M^3N}$ on the regularized empirical risk minimization problem, but it comes with consistency guarantees and finite sample generalization bounds on the discrete structured prediction loss, with constants that are polynomial in the number of parts of the structured object and do not scale as the size of the output space. More precisely, the algorithm requires a constant number of projection-oracles at every iteration, each of them having at most the same cost as the max-oracle of $\operatorname{M^3N}$. We also provide experiments on multiple tasks such as multi-class classification, ordinal regression, sequence prediction and ranking, showing the effectiveness of the algorithm. We make the following contributions:
\begin{itemize}[leftmargin=*]
    \item[-] We introduce max-min margin Markov networks ($\operatorname{M^4N}$) in \cref{def:maxminsurrogate} and prove consistency, linear calibration and finite sample generalization bounds for the regularized ERM estimator in Thms.~\ref{th:fisherconsistency}, \ref{th:comparisoninequality} and \ref{th:generalizationerm}, respectively.
    \item[-] We generalize the BCFW algorithm \cite{lacoste2012block} used for $\operatorname{M^3N}$s to $\operatorname{M^4N}$s and solve the max-min oracle iteratively with projection oracle calls using Saddle Point Mirror Prox \cite{nemirovski2004prox}. We prove bounds on the expected duality gap of the regularized ERM problem in \cref{th:gbcfw} and statistical bounds in \cref{th:generalizationalgo}.
    \item[-] In \cref{sec:experiments}, we perform a thorough experimental analysis of the proposed method on classical unstructured and structured prediction settings.
\end{itemize}

\section{Surrogate Methods for Classification}
\label{sec:surrogatemethods}
In this section, we review the first principles underlying surrogate methods starting from binary classification and moving into structured prediction. We put special attention to the difference between plug-in (e.g., logistic) and direct (e.g., SVM) classifiers to show that while there is a complete picture in the binary setting, existing direct classifiers in structured prediction lack the basic properties of binary SVMs. The first goal of this paper is to complete this picture in the structured output setting.
\subsection{A Motivation from Binary Classification}
Let $\calY=\{-1,1\}$ and $(x_1,y_1), \ldots, (x_n,y_n)$ be $n$ input-output pairs sampled from a distribution~$\rho$. The goal in binary classification is to estimate a binary-valued function~$f^\star:\calX\xrightarrow{}\calY$ that minimizes the classification error~
\begin{equation*}
\textstyle{\calE(f) = \Expect_{(x,y)\sim \rho} 1(f(x)\neq y)}.    
\end{equation*}
We can avoid working with binary-valued functions by considering instead real-valued functions $g:\calX\xrightarrow{}\Rspace{}$ and use the prediction model~$f(x) = d\circ g(x) \defeq \operatorname{sign}(g(x))$ \cite{bartlett2006convexity}
where $d$ stands for \textit{decoding}.
The resulting problem reads
\begin{equation}\label{eq:feqdog}
    g^\star\in\argmin_{g:\calX\rightarrow\Rspace{}}~\calE(d\circ g).
\end{equation}
Unfortunately, directly estimating a $g^\star$ from \eqref{eq:feqdog} is intractable for many classes of functions \cite{arora1997hardness}.

\paragraph{Convex surrogate methods. }
The source of intractability of minimizing the classification error \eqref{eq:feqdog} comes from the discreteness and non-convexity of the loss. The idea of surrogate methods \cite{bartlett2006convexity} is to consider a \textit{convex surrogate loss}~$S:\Rspace{}\times\calY\rightarrow\Rspace{}$ such that $g^\star$ can be written as
\begin{equation}\label{eq:surrogatefunction}
    g^\star = \argmin_{g:\calX\xrightarrow{}\Rspace{}}~\calR(g)\defeq \Expect_{(x,y)\sim\rho}S(g(x), y).
\end{equation}
In this case, $g^\star$ can be tractably estimated from $n$ samples over a family of functions $\calG$ using regularized ERM. The resulting estimator $g_n$ has the form
\begin{equation}\label{eq:regularizederm}
    g_n = \argmin_{g\in\calG}~\frac{1}{n}\sum_{i=1}^nS(g(x_i), y_i) + \frac{\lambda_n}{2}\|g\|_{\calG}^2,
\end{equation}
where $\lambda_n>0$ is the regularization parameter and $\|\cdot\|_{\calG}$ is the norm associated to the hypothesis space $\calG$. If not stated explicitly, our analysis of the surrogate method holds for any function space, such as reproducing kernel Hilbert spaces (RKHS) \cite{aronszajn1950theory} or neural networks \cite{lecun2015deep}, where we lose global theoretical convergence guarantees of problem \eqref{eq:regularizederm}.

The classical theoretical requirements of such a surrogate strategy are \emph{Fisher consistency} (i) and a \emph{comparison inequality} (ii):
\begin{equation*}
    \begin{array}{ll}
    \text{(i)} & \calE(f^\star) = \calE(d\circ g^\star)  \\
    \text{(ii)} & \zeta(\calE(d\circ g) - \calE(f^\star)) \leq \calR(g) - \calR(g^\star),
\end{array}
\end{equation*}
for all measurable functions $g$, where $\zeta:\Rspace{}_{+}\xrightarrow{}\Rspace{}_{+}$ is such that $\zeta(\varepsilon)\to 0$ when $\varepsilon\to 0$. Note that Condition~(i) is equivalent to \eqref{eq:feqdog}. Condition  (ii) is needed to prove consistency results, to show that $\calR(g) \rightarrow \calR(g^\star)$ implies~$\calE(d\circ g) \rightarrow \calE(f^\star)$. The existence of $\zeta$ satisfying (ii) is derived from (i) and the continuity and lower boundedness of $S(v,y)$, see Thm.~3 by \cite{zhang2004statistical}. Even though the explicit form of $\zeta$ is not needed for a consistency analysis, it is necessary to prove finite sample generalization bounds, as it is the mathematical object relating the suboptimality of the surrogate problem to the suboptimality of the original task. Note that the larger $\zeta(\varepsilon)$, the better.

\paragraph{Plug-in classifiers.}
It is known that (i) is satisfied for any function $g^\star$ that continuously depends on the conditional probability $\rho(1|x)$ as $g^\star(x) \defeq t(\rho(1|x))$,
where $t:\Rspace{}\rightarrow\Rspace{}$ is a suitable continuous bijection of the real line\footnote{It must satisfy $(u-1/2)t(u-1/2)\geq 0$ for all $u\in\Rspace{}$.}. In this case, \cref{eq:surrogatefunction} can be satisfied using \textit{smooth} losses.
Some examples are the logistic loss $\log(1 + e^{-yv})$, the squared hinge loss $\max(0, 1-yv)^2$ and the exponential loss $e^{-yv}$. In this case, the convexity and smoothness of $S(\cdot,y)$ imply that~(ii) is satisfied with $\zeta(\varepsilon) \sim\varepsilon^2$~\cite{bartlett2006convexity}. Combining this with standard convergence results of regularized ERM estimators $g_n$ on RKHS, the resulting statistical rates are of the form~$\Expect \calE(d\circ g_n) - \calE(f^\star) \sim \|g^\star\|_{\calG}n^{-1/4}$. Even if the binary learning problem is easy, $g^\star$ can be highly non-smooth away from the decision boundary, resulting in large $\|g^\star\|_{\calG}$. It is known that the dependence on the number of samples can be improved under low noise conditions \cite{audibert2007fast}.
\paragraph{Support vector machines (SVM).}
Plug-in classifiers indirectly estimate the conditional probability as $\rho(1|x)= t^{-1}(g^\star(x))$, which is more than just falling in the right binary decision set. SVMs directly tackle the classification task by estimating $g^\star\defeq f^\star=\operatorname{sign}(\rho(1|x)-1/2)$. In this case, the \textit{non-smooth} hinge loss $S(v, y) = \max(0, 1-yv)$ satisfies \eqref{eq:surrogatefunction}. Moreover, 
(ii) is satisfied with $\zeta(\varepsilon) = \varepsilon$ and statistical rates are of the form $\Expect\calE(d\circ\widehat{g}_n) - \calE(f^\star) \sim \|f^\star\|_{\calG}n^{-1/2}$.
Note that $f^\star$ is piece-wise constant on the support of $\rho$, but it can be shown $f^\star\in\calG$, (i.e., $\|f^\star\|_{\calG} < \infty$), for standard hypothesis spaces ${\calG}$ such as Sobolev spaces with input space $\Rspace{d}$ and smoothness $s > d/2$ under low noise conditions~\cite{pmlr-v75-pillaud-vivien18a}.

\subsection{Structured Prediction Setting}
In binary classification, the output data are naturally embedded in $\Rspace{}$ as $\calY=\{-1,1\}\subset\Rspace{}$. However, as this is not necessarily the case in structured prediction, it is classical \cite{taskar2005learning} to represent the output with an embedding $\phi:\calY\rightarrow\Rspace{k}$ encoding the parts structure with~$k \ll |\calY|$. Let  $g:\calX\rightarrow\Rspace{k}$ and define the following linear prediction model
 \begin{equation}\label{eq:decoding}
    \textstyle{f(x) = d\circ g(x) \defeq \argmax_{y\in\calY}~\phi(y)^\top g(x)}.
 \end{equation}
 The above decoding \eqref{eq:decoding} corresponds to the classical linear prediction model over factorized joint features $\Phi(x, y) = \phi(y)\otimes\Phi(x)$ when $g(x)$ is linear in some input features $\Phi(x)$ \cite{bakir2007predicting}. The form in \eqref{eq:decoding} is required to perform the consistency analysis but the algorithm developed in \cref{sec:rerm} can be readily extended to joint features that do not factorize.
 %The above decoding \eqref{eq:decoding} is classical in structured prediction literature \cite{bakir2007predicting} and it is assumed to be computationally tractable. 
 Non-linear prediction models have been recently proposed by \cite{belanger2016structured}, but this is out of the scope of this paper.
 
Let $L:\calY\times\calY\rightarrow\Rspace{}$ be a loss function between structured outputs encoding the cost-sensitivity of predictions. For instance, it is common to take $L$ to be the Hamming loss over the parts of the structured object. The goal in structured prediction is to estimate~$f^\star:\calX\xrightarrow{}\calY$ that minimizes the \textit{expected risk}:
 \begin{equation}\label{eq:lossminimization}
     \calE(f) = \Expect_{(x,y)\sim\rho}L(f(x), y).
 \end{equation}

\paragraph{Loss-decoding compatibility.}
It is classical to assume that the loss decomposes over the structured output parts \cite{joachims2006training}. This can be generalized as the following affine decomposition of the loss \cite{ramaswamy2013convex, nowak2019sharp}
\begin{equation}\label{eq:lossdecomposition}
    L(y, y') = \phi(y)^\top A\phi(y') + a,
\end{equation}
for a matrix $A\in \Rspace{k\times k}$ and scalar $a\in\Rspace{}$.
Indeed, assumption \eqref{eq:lossdecomposition} together with the tractability of~\eqref{eq:decoding} is essentially equivalent to the tractability of \textit{loss-augmented inference} in structural SVMs \cite{joachims2006training}. For the sake of notation, we drop the constant $a$ and work with the `centered' loss $L(y, y')-a$. We provide some examples below.

\begin{example}[Structured prediction with factor graphs]\label{ex:factorgraph}
Let $\calY=[R]^M$ be the set of objects made of $M$ parts, each in a vocabulary of size $R$.
In order to model interdependence between different parts, we consider embeddings that decompose over (overlapping) subsets of indices $\alpha\subseteq\{1,\ldots, M\}$ \cite{taskar2004max} as $\phi(y)=(\phi_{\alpha}(y_{\alpha}))_{\alpha}$. More precisely, the prediction model corresponds to
\begin{equation}\label{eq:factorgraph}
     \textstyle{\argmax_{y\in\calY}~\sum_{\alpha}\phi_\alpha(y_{\alpha})^\top v_{\alpha}},
\end{equation}
where $\phi_\alpha(y_\alpha) = e_{y_{\alpha}}\in\Rspace{R^{|\alpha|}}$ with $e_{j}$ being the $j$-th vector of the canonical basis and the dimension of the full-embedding $\phi$ is $k = \sum_{\alpha}R^{|\alpha|} \ll |\calY| = R^M$. 
It is common \cite{tsochantaridis2005large} to assume that the loss decomposes additively over the coordinates as $L(y, y') = \frac{1}{M}\sum_{m=1}L_m(y_m, y_m')$ and so the matrix $A$ associated to the loss decomposition of $L$ is low-rank. Problem \eqref{eq:factorgraph} can be solved efficiently for low tree-width structures using the junction-tree algorithm \cite{wainwright2008graphical}. More specifically, if the objects are sequences with embeddings modelling individual and adjacent pairwise characters, Problem \eqref{eq:factorgraph} can be solved in time $O(MR^2)$ using the Viterbi algorithm \cite{viterbi1967error}.
\end{example}

\begin{example}[Ranking and matching] \label{ex:matching} The output space is the group of permutations $\calS_M$ acting on~$\{1, \ldots, M\}$. This setting also includes the task of matching the nodes of two graphs of the same size \cite{caetano2009learning}. We represent a permutation $\sigma\in \calS_M$ using the corresponding permutation matrix $\phi(\sigma) = P_{\sigma}\in\Rspace{M\times M}$. The prediction model corresponds to the linear assignment problem \cite{burkard2012assignment}
\begin{equation}\label{eq:linearassignment}
    \textstyle{\argmax_{\sigma\in\calS_M}~\langle P_{\sigma}, v\rangle_{F}},
\end{equation}
where $v\in \Rspace{M\times M}$, $\langle\cdot, \cdot\rangle_{F}$ is the Frobenius scalar product and $k=M^2 \ll |\calY|=M!$. The Hamming loss on permutations satisfies \cref{eq:lossdecomposition} as $L(\sigma, \sigma') = \frac{1}{M}\sum_{m=1}^M1(\sigma(m)\neq \sigma'(m)) = 1 - \frac{1}{M}\langle P_{\sigma}, P_{\sigma'}\rangle_{F}$.
The linear assignment problem \eqref{eq:linearassignment} can be solved in time  $O(M^3)$ using the Hungarian algorithm \cite{kuhn1955hungarian}.
\end{example}

\paragraph{Plug-in classifiers in structured prediction.}
Let $\mu(x) = \Expect_{y\sim\rho(\cdot|x)}\phi(y)$ be the conditional expectation of the output embedding.
Using the fact that $f^\star$ can be characterized pointwise in $x$ as the minimizer in $y$ of $\phi(y)^\top A\mu(x)$ \cite{nowak2019sharp}, % \footnote{In general $\fstar$ is not unique. We assume we have a method to choose a unique output between the optimal ones, which is always possible using an ordering of the discrete set $\calY$.}
it directly follows that (i) is satisfied for $g^\star(x) = -A\mu(x)$ and, analogously to binary classification, it can be estimated using \textit{smooth} surrogates. Some examples are the quadratic surrogate (QS)~$\|v+A\phi(y)\|_2^2$ \cite{ciliberto2016consistent} that estimates $g^\star$ and conditional random fields (CRF) \cite{Lafferty:2001:CRF:645530.655813} defined by $\log\big(\sum_{y'\in\calY}\exp v^\top \phi(y')\big) - v^\top \phi(y)$ that estimate an invertible continuous transformation of $\mu(x)$. Although CRFs have a powerful probabilistic interpretation, they cannot incorporate the cost-sensitivity matrix $A$ into the surrogate loss, and it must be added a posteriori in the decoding \eqref{eq:decoding} to guarantee consistency.
It was shown by \cite{nowak2019general} that these methods satisfy condition (ii) with $\zeta(\varepsilon) \sim \varepsilon^2$ and achieve the analogous statistical rates of binary plug-in classifiers $\sim \|g^\star\|_{\calG} n^{-1/4}$.

\paragraph{SVMs for structured prediction.}
The extension of binary SVM to structured outputs is the structural SVM (SSVM) \cite{joachims2006training} (denoted $\operatorname{M^3N}$s \cite{taskar2004max} in the factor graph setting described in \cref{ex:factorgraph}). It corresponds to the following surrogate loss
\begin{equation}\label{eq:structuralsvm1}
    S(v, y) = \max_{y'\in\calY}~\phi(y)^\top A\phi(y') + v^\top \phi(y') - v^\top \phi(y).
\end{equation}
In the multi-class case with $\calY=\{1,\ldots,k\}$ and $L(y,y')=1(y\neq y')$ it is also known as the Crammer-Singer SVM (CS-SVM) \cite{crammer2001algorithmic} and reads $S(v,j) = \max_{r\neq j}~1 + v_r - v_j$. It shares some properties of the binary SVM such as the upper bound property, i.e., $L(d\circ v, y) \leq S(v, y)$ for all~$y\in\calY$. However, an important drawback of this loss is that while the upper bound property holds, the minimizer of the surrogate expected risk $g^\star$ and the one of the expected risk $f^\star$ do not coincide when the problem is far from deterministic, as shown by the following \cref{th:inconsistencysvm}.
\begin{proposition}[Inconsistency of CS-SVM \cite{liu2007fisher}]\label{th:inconsistencysvm} The CS-SVM is Fisher-consistent if and only if for all $x\in\calX$, there exists $y\in\{1,\ldots, k\}$ such that $\rho(y|x) > 1/2$.
\end{proposition}
Note that the consistency condition from \cref{th:inconsistencysvm} is much harder to be met in the structured prediction case as the size of the output space is exponentially large, and it is always satisfied in the binary case (the binary SVM is always consistent). Although there exist consistent extensions of the SVM to the cost-sensitive multi-class setting such as the ones from \cite{lee2004multicategory, mroueh2012multiclass}, they cannot be naturally extended to the structured setting. In the following section we address this problem by introducing the max-min surrogate and studying its theoretical properties.
\section{Max-Min Surrogate Loss} \label{sec:maxmin}

Assume that the loss is not degenerated, i.e., $L(y,y)<L(y,y')$ for all $y,y'\in\calY$ such that $y \neq y'$. In this case,~$f^\star(x)$ is the minimizer in $y$ of $\phi(y)^\top A^\top \phi(f^\star(x))$, which means that \eqref{eq:feqdog} is satisfied by
\begin{equation*}
    g^\star(x) \defeq -A^\top \phi(f^\star(x)) \in \Rspace{k}.
\end{equation*}
Note the analogy with SVMs, where we directly estimate~$f^\star$ but now through the representation~$\phi$ of the structured output, avoiding the full enumeration of $\calY$. We need to find a surrogate function $S(v,y)$ that satisfies \cref{eq:surrogatefunction} for this~$g^\star$.
Following the same notation as \cite{nowak2019general}, we define the \textit{marginal polytope} \cite{wainwright2008graphical} as the convex hull of the embedded output space~$\calM=\operatorname{hull}(\phi(\calY))\subset \Rspace{k}$. 
\begin{definition}[Max-min surrogate loss]\label{def:maxminsurrogate}
Define the max-min loss as
\begin{equation}\label{eq:maxminsurrogate}
   S(v, y) \defeq \max_{\mu\in\calM}\min_{y'\in\calY}~\phi(y')^\top A\mu + v^\top \mu - v^\top \phi(y).
\end{equation}
\end{definition}
The max-min loss is \emph{non-smooth}, \emph{convex} and can be cast as a Fenchel-Young loss \cite{blondel2019learning}. More specifically, \cref{eq:maxminsurrogate} can be written as $S(v,y) = \Omega^*(v) + \Omega(\phi(y)) - v^\top \phi(y)$ with 
\begin{equation}\label{eq:Omega}
\Omega(\mu) = -\min_{y'\in\calY}~\phi(y')^\top A\mu + 1_{\calM}(\mu),    
\end{equation}
where $\Omega(\phi(y))=0$ for all $y\in\calY$, $\Omega^*$ denotes the Fenchel-conjugate of $\Omega$, and $1_{\calM}(\mu)=0$ if $\mu\in\calM$ and $+\infty$ otherwise.

 Note that the dependence on $y$ is only in the linear term~$v^\top \phi(y)$, while for SSVMs \eqref{eq:structuralsvm1} it appears in the maximization. Thus, we can study the geometry of the loss through the non-smooth convex function $\Omega^*(v)$ (see \cref{fig:geometries} for visualizations of some representative unstructured examples). Connections between surrogates \eqref{eq:maxminsurrogate} and \eqref{eq:structuralsvm1} are discussed in \cref{sec:comparisonSSVM}.
 
 \subsection{Fisher Consistency} \label{sec:fisherconsistency}
 Fisher consistency is provided by the following \cref{th:fisherconsistency}.
\begin{theorem}[Fisher Consistency (i)]\label{th:fisherconsistency}
The surrogate loss~(\ref{eq:maxminsurrogate}) satisfies (i) for $g^\star(x) = -A^\top \phi(f^\star(x))$.
\end{theorem}
This result has been proven by \citet{fathony2018consistent} in the cost-sensitive multi-class case. Our proof of \cref{th:fisherconsistency} is constructive and based on Fenchel duality.
\begin{proof}[Sketch of the proof] We want to show that $-A^\top \phi(f^\star(x))$ is the minimizer of $\Expect_{y\sim \rho(\cdot|x)}~S(v, y)$ almost surely for every~$x$. The proof is constructive and based on Fenchel duality, using the Fenchel-Young loss representation of the max-min surrogate. First, note that the conditional surrogate risk can be written as $\Expect_{y\sim \rho(\cdot|x)}~S(v, y) = \Omega^*(v) - v^\top\mu(x)$, where $\mu(x)=\Expect_{y\sim\rho(\cdot|x)}\phi(y)\in\calM$. Second, note that by Fenchel-duality,
$\partial_{\mu}\Omega(\mu(x))$ is the set of minimizers of~$\Omega^*(v) - v^\top\mu(x)$. Finally, if we assume that the set of~$x\in\calX$ such that $\mu(x)$ is in the boundary of $\calM$ has measure zero, then~
\begin{equation*}
    -A^\top f^\star(x) \in \partial_{\mu}\Omega(\mu(x)),
\end{equation*}
where $\Omega$ is defined in~\eqref{eq:Omega} and we have used that $f^\star(x)$ is the minimizer in $y$ of $\phi(y)^\top A\mu(x)$. A more detailed proof can be found in \cref{app:fisherconsistency}.
\end{proof}

\begin{figure*}[ht!]
    \centering
    \hspace{-0.01\textwidth}
    \resizebox{1.018\textwidth}{!}{
    \includegraphics[width=0.0475\textwidth]{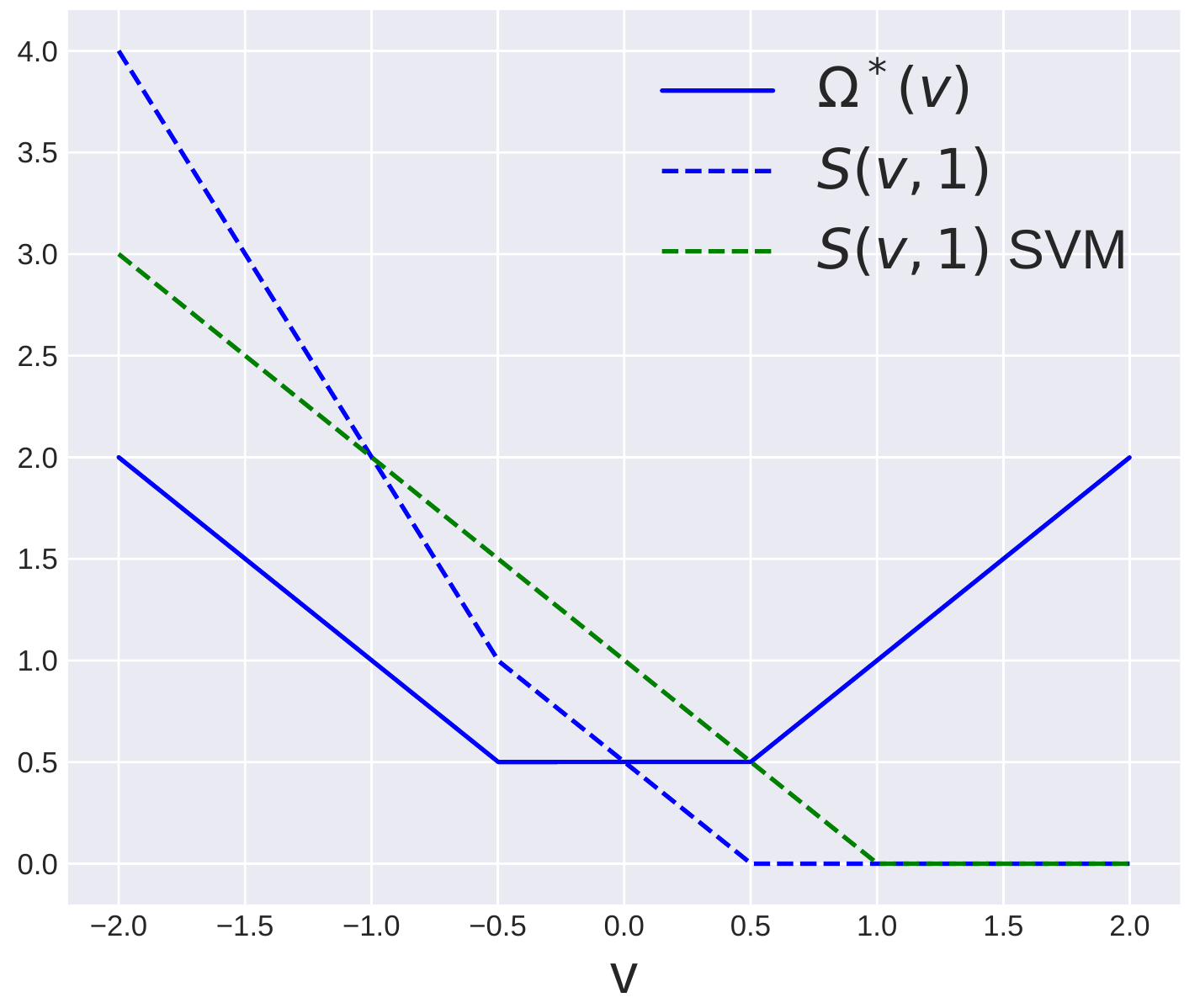}
    \includegraphics[width=0.04\textwidth]{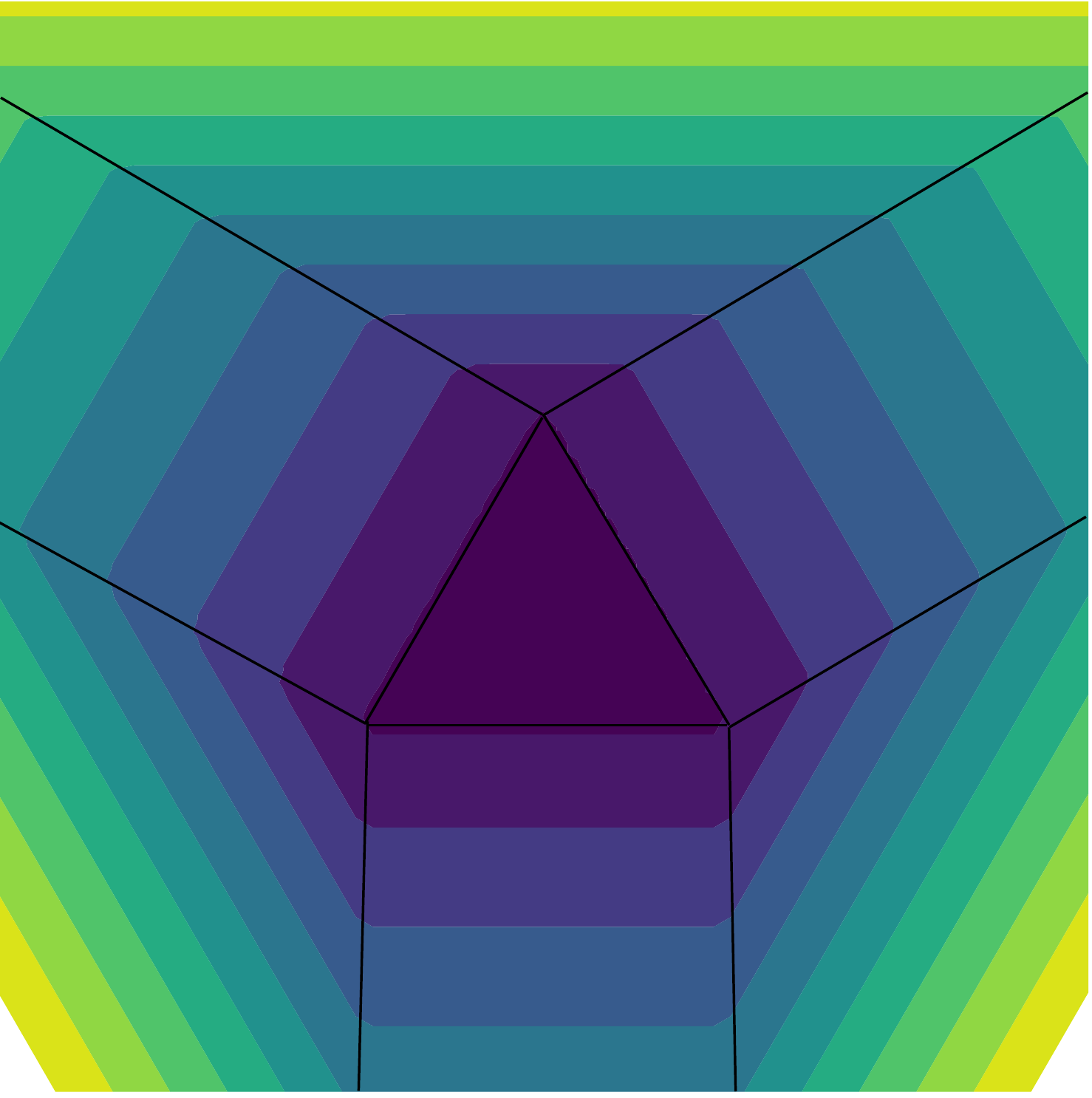}
    \includegraphics[width=0.04\textwidth]{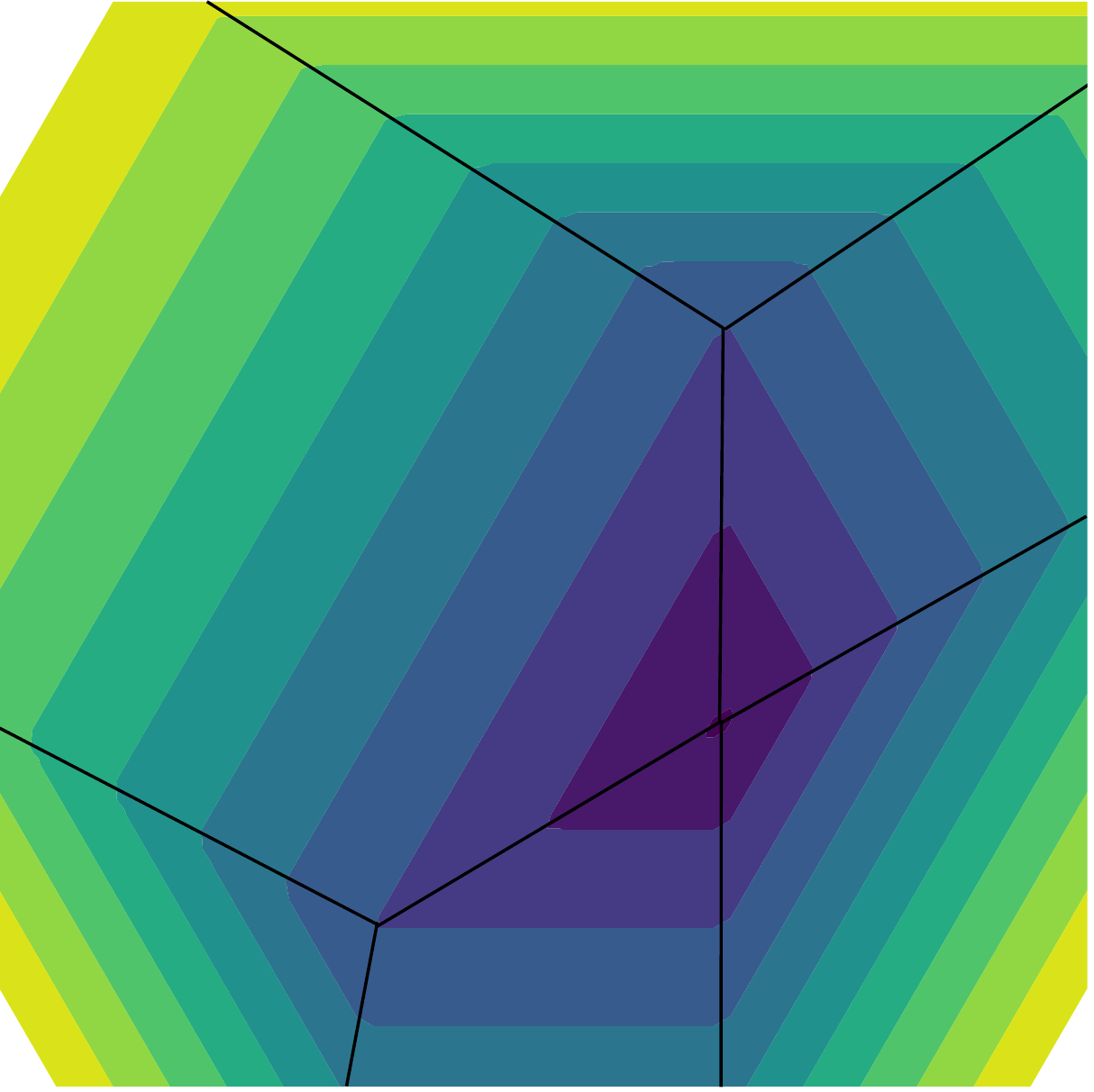}
    }
    \caption{\textbf{Left:} The binary max-min loss has two symmetric kinks instead of one as the SVM. \textbf{Middle:} $\Omega^*(v)$ in $v^\top 1=0$ for multi-class 0-1 loss $1(y\neq y')$ with $k=3$. \textbf{Right:} $\Omega^*(v)$ in $v^\top 1=0$ for ordinal regression with the absolute loss $|y-y'|$ with $k=3$.}
    \label{fig:geometries}
    \vspace{-0.5cm}
\end{figure*}
 \subsection{Comparison Inequality} \label{sec:comparisoninequality}
 Fisher consistency is not enough to prove finite-sample generalization bounds on the excess risk $\calE(d\circ g) - \calE(f^\star)$. For this, we provide in the following \cref{th:comparisoninequality} an explicit form of the comparison inequality.
\begin{theorem}[Comparison inequality (ii)]\label{th:comparisoninequality} Assume $L$ is symmetric and that there exists $C > 0$ such that for any probability $\alpha \in \Delta_{\calY}$, it holds that $\alpha_y \geq 1/C$ for~$y \in \argmin_{y \in \calY} \mathbb{E}_{z \sim \alpha} L(y,z)$. Then, the comparison inequality (ii) for the max-min loss \eqref{eq:maxminsurrogate} reads
\begin{equation*}
\calE(d\circ g) - \calE(f^\star) \leq C(\calR(g) - \calR(g^\star)).
\end{equation*}
\end{theorem}
The second condition on the loss states that if $y$ is optimal for $x$, then its conditional probability is bounded away from zero as $\rho(y|x)\geq 1/C$. This condition is used to obtain a simple quantitative lower bound on the function $\zeta$ of~(ii) and more tight (albeit less explicit in general) expressions of the constant $C$ can be found in Appendices C.3 and C.4.

\paragraph{Constant $\boldsymbol{C}$ for multi-class.} When $L(y,y')=1(y\neq y')$ with $\calY=\{1,\ldots,k\}$, we have that $C=k$, as the minimum conditional probability of an optimal output is $1/k$. The constant for this specific setting was derived independently using a different analysis by \citet{duchi2018multiclass}.

\paragraph{Constant $\boldsymbol{C}$ for factor graphs (\cref{ex:factorgraph}).} For a factor graph with separable embeddings and a decomposable loss $L=\frac{1}{M}\sum_{m=1}^ML_m(y_m, y_m')$, we have that~$C=\max_{m\in[M]}C_m$, where $C_m$ is the constant associated to the individual loss $L_m$. This is proven in \cref{prop:decomposablelosses}.

\paragraph{Constant $\boldsymbol{C}$ for ranking and matching (\cref{ex:matching}).} In this setting, \cref{th:comparisoninequality} gives $C=M!$, and so the relation between both excess risks is not informative. The problem of exponential constants in the comparison inequality was pointed out by \citet{osokin2017structured}. We can weaken the assumption and change condition $\alpha_{y}\geq 1/C$ to
\begin{equation*}
    \max_{\beta\in\Delta_{\calY}} \beta_{y}\hspace{0.2cm}\text{s.t.}\hspace{0.2cm} \Expect_{z\sim\beta}\phi(z) = \Expect_{z'\sim\alpha}\phi(z') \geq 1/C.
\end{equation*}
Under this assumption, we have that $C=M$, thus avoiding the exponentially large size of the output space.

\subsection{Generalization of Regularized ERM}\label{sec:generalizationERM}
In the following \cref{th:generalizationerm}, we use this result to prove a finite-sample generalization bound on the regularized ERM estimator \eqref{eq:regularizederm} when the hypothesis space $\calG$ is a vector-valued RKHS.
\begin{theorem}[Generalization of regularized ERM]\label{th:generalizationerm} 
Let $\calG$ be a vector-valued RKHS, assume $g^\star\in\calG$ and let $g_n$ and~$\lambda_n = \kappa L\log^{1/2}(1/\delta)n^{-1/2}$ as in \eqref{eq:regularizederm}. Then, with probability $1-\delta$:
\begin{equation*}
    \calE(d\circ g_n) - \calE(f^\star) \leq M \|\phi(f^\star)\|_{\calG}\sqrt{\frac{\log(1/\delta)}{n}},
\end{equation*}
with $M =  \kappa C L\|A\|$. Here, $L = 2\max_{y\in\calY}\|\phi(y)\|_2$, $\|A\| = \sup_{\|v\|_2 \leq 1}\|Av\|_2$, $\kappa=\sup_{x\in\calX}\operatorname{Tr}K(x, x)^{1/2}$ is the size of the features  and $C$ is the one of \cref{th:comparisoninequality}.
\end{theorem}
Analogously to the binary case, the multivariate function~$\phi(f^\star)$ is piecewise constant on the support of the distribution $\rho$. In Theorem~\ref{thm:conditions-fstar-in-G} in Appendix~\ref{app:generalization} we prove that standard low noise conditions, analogous to the one discussed by \citet{pmlr-v75-pillaud-vivien18a} for the binary case, are enough to guarantee $\|\phi(f^\star)\|_{\calG} < \infty$.

\section{Comparison with Structural SVM}
\label{sec:comparisonSSVM}
\paragraph{Max-min as a correction of the Structural SVM.}
We can re-write the maximization over the discrete output space $\calY$ in the definition of the SSVM \eqref{eq:structuralsvm1} as a maximization over its convex hull $\calM=\operatorname{hull}(\phi(\calY))$
\begin{equation}\label{eq:structuralsvm}
    \textstyle{S(v, y) = \max_{\mu\in\calM}~\phi(y)^\top A\mu + v^\top \mu - v^\top \phi(y)}.
\end{equation}
Note the similarity between \eqref{eq:maxminsurrogate} and \eqref{eq:structuralsvm}. In particular, the max-min loss differs from the structural SVM in that the maximization is done using $\min_{y'\in\calY}\phi(y')^\top A\mu$ and not the loss at the observed output $y$ as $\phi(y)^\top A\mu$. Hence, we can view the max-min surrogate as a \textit{correction} of the SSVM so that basic statistical properties (i) and (ii) hold. Moreover, this connection might be used to properly understand the statistical properties of SSVM. This is left for future work.

\paragraph{Notion of max-min margin.}
Given $v\in\Rspace{k}$ and $y_i\in\calY$, the classical SSVM is motivated by a soft version of the following notion of margin:
\begin{equation*}\label{eq:maxmargin}
    v^\top \phi(y_i) - v^\top \phi(y) \geq L(y_i, y) = \phi(y_i)^\top A\phi(y),
\end{equation*}
for all $y\in\calY$, which is equivalent to $v^\top \phi(y_i) - v^\top \mu \geq \phi(y_i)^\top A\mu$ for all $\mu\in\calM$.
However, we have seen in \cref{th:inconsistencysvm} that this condition is too strong and only leads to a consistent method if the problem is nearly deterministic, i.e., we observe the optimal~$y$ with large probability, which, as already mentioned, is generally far from true in structured prediction. The max-min surrogate \eqref{eq:maxminsurrogate} deals with the case where this strong condition is not met and works with a notion of margin that compares groups of outputs instead of just pairs. We define the max-min margin as
\begin{equation}\label{eq:maxminmargin}
    \textstyle{v^\top \phi(y_i) - v^\top \mu \geq \min_{y'\in\calY}\phi(y')^\top A\mu,}
\end{equation}
for all $\mu\in\calM$.
After introducing slack variables in \eqref{eq:maxminmargin} we obtain a soft version of the max-min margin that leads to the max-min regularized ERM problem \eqref{eq:regularizederm}.
\section{Algorithm} \label{sec:rerm}

\begin{figure*}[t]
\begin{minipage}{0.4\linewidth}
\begin{algorithm}[H]
\small
   \caption{GBCFW \vspace{0.08cm} (primal)}
   \label{alg:gbcfw}
\begin{algorithmic}
   \STATE Let $ w^{(0)}\defeq  w_i^{(0)}\defeq 0$
   \FOR{$t=0$ {\bfseries to} $T$}
   \STATE Pick $i$ at random in $\{1,\ldots, n\}$
   \STATE $(\mu_i^\star,\nu_i^\star)\in \calO_{K}(g_{w^{(t)}}(x_i), \mu_i^\star,\nu_i^\star)$
   \STATE $ w_s \defeq \Phi(x_i)(\mu_i^\star-\phi(y_i))^\top/(\lambda n)$ \\
   \STATE $ w_i^{(t+1)} \defeq (1-\frac{2n}{t+2n}) w_i^{(t)} + \frac{2n}{t+2n} w_s$
   \STATE $ w^{(t+1)}\defeq  w^{(t)} +  w_i^{(t+1)} -  w_i^{(t)}$
   \ENDFOR
\end{algorithmic}
\end{algorithm}
\end{minipage}
\begin{minipage}{0.58\linewidth}
\begin{algorithm}[H]
\small
   \caption{SP-MP \quad $(\bar{\mu}^{(K)}, \bar{\nu}^{(K)}) \in \calO_K(v, \mu^{(0)} \nu^{(0)})$}
   \label{alg:spmp}
\begin{algorithmic}
   \vspace{0.1cm}
   \FOR{$k=0$ {\bfseries to} $K-1$}
   \STATE $\mu_{1/2}^{(k+1)}\in{\argmin_{\mu\in\calM}}-\eta \mu^\top (A^\top \nu^{(k)}+v) + D_{-H}(\mu, \mu^{(k)})$ \\
   \STATE $\nu_{1/2}^{(k+1)}\in{\argmin_{\nu\in\calM}}~\eta \nu^\top A\mu^{(k)} + D_{-H}(\nu, \nu^{(k)})$ \\
   \STATE $\mu^{(k+1)}\in{\argmin_{\mu\in\calM}}-\eta \mu^\top (A^\top \nu_{1/2}^{(k+1)}+v) + D_{-H}(\mu, \mu^{(k)})$ \\
   \STATE $\nu^{(k+1)}\in{\argmin_{\nu\in\calM}}~\eta \nu^\top A\mu_{1/2}^{(k+1)} + D_{-H}(\nu, \nu^{(k)})$ \\
   \ENDFOR
   \STATE $\textstyle{\bar{\mu}^{(K)} \defeq \frac{1}{K}\sum_{k=1}^{K}\mu^{(k)}}$, $\textstyle{\bar{\nu}^{(K)} \defeq \frac{1}{K}\sum_{k=1}^{K}\nu^{(k)}}$
\end{algorithmic}
\end{algorithm}
\end{minipage}
\end{figure*}

In this section we derive a dual-based algorithm to solve the max-min regularized ERM problem \eqref{eq:regularizederm} when the hypothesis space is a RKHS. The algorithm can be easily adapted to the case where $g$ is parametrized using a neural network as commented at the end of \cref{sec:computationoracle}.

\subsection{Problem Formulation} 
Let $\calG\subset\{g:\calX\rightarrow\Rspace{k}\}$ be a vector-valued RKHS, which we assume of the form $\calG=\Rspace{k}\otimes\overline{\calG}$, where $\overline{\calG}$ is a scalar RKHS with associated features $\Phi:\calX\rightarrow\overline{\calG}$. Every function in $\calG$ can be written as~$g_w(x) = w^\top\Phi(x)\in\Rspace{k}$ where~$w_j, \Phi(x)\in\overline{\calG}$. For the sake of presentation, we assume that $\calG=\Rspace{d\times k}$ is finite dimensional, but our analysis also holds for the infinite dimensional case. The dual~\textbf{(D)} of the regularized ERM problem \eqref{eq:regularizederm} for the max-min surrogate loss~\eqref{eq:maxminsurrogate} reads
\begin{align*}\label{eq:dualmaxmin}
  \textbf{(D)} \hspace{0.4cm} & \max_{  \mu\in\calM^n}~ \frac{1}{n}\sum_{i=1}^n \min_{y'}\phi(y')^\top A\mu_i -\frac{\lambda}{2}\|\Phi_n(\mu-\phi_n)\|_2^2,
\end{align*}
where $\Phi_n = \frac{1}{\lambda n}(\Phi(x_1),\ldots,\Phi(x_n))$ is the $d\times n$ scaled input data matrix and $\phi_n = (\phi(y_1),\ldots,\phi(y_n))^\top $ is the $n\times k$ output data matrix. The dual variables map to the primal variables through the mapping $w(\mu) = \frac{1}{\lambda n}\sum_{i=1}^n\Phi(x_i)(\mu_i - \phi(y_i))^\top$. By strong duality, it holds~$w^\star = w(\mu^\star)$.
The dual formulation $\textbf{(D)}$ is a constrained \textit{non-smooth} optimization problem, where the non-smoothness comes from the first term of the objective function. In order to derive a learning algorithm, we leverage ideas from the block-coordinate Frank-Wolfe algorithm used for SSVMs.

\subsection{Derivation of the Algorithm}
\paragraph{Background on BCFW for $\boldsymbol{\operatorname{M^3Ns}}$.}
The dual of the SSVM is the same as problem \textbf{(D)} but the first term is linear: $\frac{1}{n}\sum_{i=1}^n\phi(y_i)^\top A\mu_i$, making the dual objective function \textit{smooth}. The BCFW algorithm~\cite{lacoste2012block} minimizes a linearization of the smooth dual objective function block-wise, using the separability of the compact domain. At each iteration $t$, the algorithm picks $i\in[n]$ at random, and updates $\mu_{i}^{(t+1)} = (1 - \gamma)   \mu_{i}^{(t)} + \gamma \bar{  \mu}_{i}^{(t+1)}$ with~$\bar{\mu}_{i}^{(t+1)} \defeq \argmax_{\mu_i'\in\calM}~\langle  \mu_i', \nabla_{(i)} h(\mu^{(t)})\rangle$ where $h$ is the dual objective and $\gamma$ is the step-size.
Note that $\bar{\mu}_{i}^{(t+1)}$ is an extreme point of $\calM$ and it can be written as a combinatorial maximization problem over $\calY$ that corresponds precisely to inference \eqref{eq:decoding}. In the next subsection, we generalize BCFW to the case where the dual is a sum of a non-smooth and a smooth function such as the dual $\textbf{(D)}$ of our problem.

\paragraph{Generalized BCFW (GBCFW) for $\boldsymbol{\operatorname{M^4N}}$.}
Borrowing ideas from \citet{bach2015duality} in the non block-separable case, we only linearize the smooth-part of the function, i.e., the quadratic term. We change the computation of the direction~to
\begin{align*}
     \bar{\mu}_{i}^{(t+1)} &\defeq \argmax_{  \mu_i'\in\calM}\langle  \mu_i', \nabla_{(i)} \frac{-\lambda}{2}\|\Phi_n(\mu^{(t)}-\phi_n)\|_2^2\rangle \\
     &+ \min_{y'}\phi(y')^\top A\mu_i'
     = \calO(g_{w(\mu^{(t)})}(x_i)),
\end{align*}
where the max-min oracle $\calO:\Rspace{k}\xrightarrow{}\calM$ is defined as
\begin{equation}\label{eq:maxminoracle}
    \textstyle{\calO(v) = \argmax_{\mu\in\calM}\min_{\nu\in\calM}~\nu^\top A\mu + v^\top \mu.}
\end{equation}
Note that the mapping $w(\mu)$ between primal and dual variables is affine. Hence, one can write the update of the primal variables without saving the dual variables as detailed in \cref{alg:gbcfw}.
The following \cref{th:gbcfw} specifies the required number of iterations of \cref{alg:gbcfw} to obtain an~$\varepsilon$-optimal solution with an approximate oracle \eqref{eq:maxminoracle}.
\begin{theorem}[Convergence of GBCFW with approximate oracle]\label{th:gbcfw}
Let $\varepsilon >0$. If the approximate oracle provides an answer with error $\varepsilon/2$, then the final error of \cref{alg:gbcfw} achieves an expected duality gap of $\varepsilon$ when $T = \tilde{O}\big(n + \frac{2R^2}{\lambda\varepsilon} \operatorname{diam}(\calM)^2\big)$, 
where $R$ is the maximum norm of the features.
\end{theorem}
\subsection{Computation of the Max-Min Oracle}\label{sec:computationoracle}
The max-min oracle \eqref{eq:maxminoracle} corresponds to a concave-convex bilinear saddle-point problem.
We use a standard alternating procedure of ascent and descent steps on the variables $\mu$ and $\nu$, respectively.
Consider a strongly concave differentiable entropy $H:\calC\supset\calM\rightarrow\Rspace{}$ defined in a convex set $\calC$ containing $\calM$ such that $\nabla H(\calC)=\Rspace{k}$ and~$\lim_{\mu\in\partial\calC}\|\nabla H(\mu)\|=+\infty$, where $\partial\calC$ is the boundary of~$\calC$. Then, perform Mirror ascent/descent updates using $-H$ as the Mirror map. For instance, if $u=A^\top\nu + v$ is the gradient of \eqref{eq:maxminoracle} w.r.t $\mu$, the update on $\mu$ takes the following form:
\begin{equation}\label{eq:projection}
    \textstyle{\argmin_{\mu\in\calM}-\eta \mu^\top u + D_{-H}(\mu, \mu^{(t)})}, 
\end{equation}
where $D_{-H}(\mu, \mu') = -H(\mu) + H(\mu') + \nabla H(\mu')^\top (\mu - \mu')$ is the Bregman divergence associated to the convex function $-H$. The resulting ascent/descent algorithm has a convergence rate of $O(t^{-1/2})$, which can be considerably improved to $O(t^{-1})$ with essentially no extra cost by performing four projections instead of two at each iteration. This corresponds to the extra-gradient strategy, called \textit{Saddle Point Mirror Prox} (SP-MP) when using a Mirror map and is detailed in \cref{alg:spmp}.

\paragraph{Projection for factor graphs (\cref{ex:factorgraph}).} The entropy in $\calM$ defined by the factor graph \cite{wainwright2008graphical} can be written explicitly in terms of the entropies of each part $\alpha\subset[M]$ if the factor graph has a junction tree structure \cite{koller2009probabilistic}. For instance, in the case of a sequence of length $M$ with unary and adjacent pairwise factors, we have~$\textstyle{H(\mu) = \sum_{m=1}^{M-1}H_S(\mu_{m, m+1}) - \sum_{m=1}^MH_S(\mu_m)}$, where $H_S$ is the Shannon entropy and $\mu_m, \mu_{m,m+1}$ are the unary and pair-wise marginals, respectively.
The projection \eqref{eq:projection} corresponds to marginal inference in CRFs and can be computed using the sum-product algorithm in time $O(MR^2)$. In this case, the complexity of the projection-oracle is the same the one of the max-oracle for SSVMs.

\paragraph{Projection for ranking and matching (\cref{ex:matching}).} In this setting, the projection using the entropy in $\calM$ is known to be \#P-complete \cite{valiant1979complexity}. Thus, CRFs are essentially intractable in this setting \cite{petterson2009exponential}. If instead we use the entropy $\textstyle{H(P) = - \sum_{i, j=1}^MP_{ij}\log P_{ij}}$ defined over the marginals $P\in\calM$, the projection can be computed up to precision $\delta$ in $O(M^2/\delta)$ iterations using the Sinkhorn-Knopp algorithm \cite{cuturi2013sinkhorn}. This can be potentially much cheaper than the max-oracle of SSVMs, which has a cubic dependence in $M$. The projection with respect to the Euclidean norm has similar complexity but implementation is more involved \cite{blondel2017smooth}.

\paragraph{Warm-starting the oracles.}
On the one hand, \cref{alg:gbcfw} is guaranteed to converge as long as the error incurred in the oracle $\calO$ decreases sublinearly with the number of global iterations as $\varepsilon_t \propto n / (t + n)$ (see \cref{app:generalconvergenceresult}). On the other hand, \cref{alg:spmp} can be naturally warm-started because it is an \textit{any-time} algorithm as the step-size $\eta$ does not depend on the current iteration or a finite horizon. Hence, we are in a setting where a warm-start strategy can be advantageous. More specifically, at every iteration $t$, we save the pairs $(\mu_i^\star, \nu_i^\star)\in \calO(g_{w^{(t)}}(x_i))$ and the next time we revisit the $i$-th training example we initialize \cref{alg:spmp} with this pair. Even though the formal demonstration of the effectiveness of the strategy is technically hard, we provide a strong experimental argument showing that a constant number of \cref{alg:spmp} iterations are enough to match the allowed error~$\varepsilon_t$.

\paragraph{Using the kernel trick.} An extension to infinite-dimensional RKHS is straightforward to derive as \cref{alg:gbcfw} is dual-based. In this case, the algorithm keeps track of the dual variables $\mu_i$ for $i=1,\ldots,n$.

\paragraph{Connection to stochastic subgradient algorithms.} It is known that (generalized) conditional gradient methods in the dual are formally equivalent to subgradient methods in the primal \cite{bach2015duality}. Indeed, note that $w_s$ in \cref{alg:gbcfw} is a subgradient of the scaled surrogate loss $S(g_w(x_i), y_i) / \lambda n$. However, the dual-based analysis we provide in this paper allows us to derive guarantees on the expected duality gap and a line-search strategy, which we leave for future work. Viewing \cref{alg:gbcfw} as a subgradient method is useful when learning the data representation with a neural network. More specifically, both \cref{alg:gbcfw} and \cref{alg:spmp} remain essentially unchanged by applying the chain rule in the update of $w$.

\subsection{Statistical Analysis of the Algorithm} \label{sec:statisticalanalysis}

Finally, the following \cref{th:generalizationalgo} shows that the full algorithm without the warm-start strategy achieves the same statistical error as the regularized ERM estimator \eqref{eq:regularizederm} after at most $O(n\sqrt{n})$ projections oracle calls.
\begin{theorem}[Generalization bound of the algorithm]\label{th:generalizationalgo} Assume the setting of \cref{th:generalizationerm}. Let $g_{n,T}$ be the $T$-th iteration of \cref{alg:gbcfw} applied to problem \eqref{eq:regularizederm}, where each iteration is computed with $K=O(\sqrt{n})$ iterations of \cref{alg:spmp}. Then, after $T = O(n)$ iterations, $g_{n,T}$ satisfies the bound of \cref{th:generalizationerm} with probability~$1-\delta$. 
\end{theorem}

As we will show in the next section, in practice a constant number of iterations of \cref{alg:spmp} are enough when using the warm-start strategy. Hence, the total number of required projection-oracles is $O(n)$.

\section{Experiments} \label{sec:experiments}

We perform a comparative experimental analysis for different tasks between $\operatorname{M^4N}$s, $\operatorname{M^3N}$s and $\operatorname{CRF}$s optimized with Generalized BCFW + SP-MP (\cref{alg:gbcfw} + \cref{alg:spmp}), BCFW \cite{lacoste2012block} and SDCA \cite{shalev2013stochastic}, respectively. All methods are run with our own implementation \footnote{Code in \url{https://github.com/alexnowakvila/maxminloss}}. We use datasets of the UCI machine learning repository \cite{asuncion2007uci} for multi-class classification and ordinal regression, the OCR dataset from \citet{taskar2004max} for sequence prediction and the ranking datasets used by \citet{korba2018structured}. We use 14 random splits of the dataset into 60\% for training, 20\% for validation and 20\% for testing. We choose the regularization parameter $\lambda$ in $\{2^{-j}\}_{j=1}^{10}$ using the validation set and show the average test loss on the test sets in \cref{table:experiments} of the model with the best $\lambda$. We use a Gaussian kernel and perform $50$ passes on the data and set the number of iterations of \cref{alg:spmp} to $20$ and $10$ times the length of the sequence for sequence prediction. The results are in \cref{table:experiments}. We perform better than $\operatorname{M^3N}$s in most of the datasets for multi-class classification, ordinal regression and ranking, while we obtain similar results in the sequence dataset with the three methods.

\paragraph{Effect of warm-start. } We perform an experiment tracking the test loss and the average error in the max-min oracle for different iterations of \cref{alg:spmp} with and without warm-starting. The experiments are done in two datasets for ordinal regression and they are shown in Table 2. We observe that both the test loss and average oracle error are lower for the warm-start strategy. Moreover, when warm-starting the final test error barely changes when increasing the iterations past the 50 iteration threshold.

\begin{table}[!h]\label{table:table_experiments}
    \resizebox{!}{0.51\linewidth}{
    \begin{tabular}{@{}llllll@{}}
%%%% MULTILABEL
\toprule

Task
& Dataset  & $(d, n, k)$ & $\operatorname{M^3N}$ & $\operatorname{CRF}$ & $\operatorname{M^4N}$ \\ 
\midrule

\multirow{7}{*}{MC}
& segment & (19, 2310, 7) & 6.64\% & 6.43\% & \textbf{6.09\%} \\
& iris & (4, 150, 3)  & 3.33\% & 3.08\% & 3.33\% \\
& wine & (13, 178, 3) & 2.56\% & 2.14\% & \textbf{2.35\%} \\
% & letter & (16, 15000, 26) & - & - & - \\
 & vehicle & (18, 846, 4) & 24.6\% & 25.1\% & \textbf{24\%} \\
& satimage & (36, 4435, 6) & 12.2\% & 11.5\% & \textbf{11.9\%} \\
& letter & (16, 15000, 26) & 14.6\% & 13.2\% & \textbf{13.5\%} \\
% & kr-vs-k & (6, 28056, 18) & \textbf{5.52E-01} & 5.47E-01 & 5.59E-01 \\
& mfeat & (216, 2000, 10) & \textbf{3.94\%} & 4.35\% & 3.96\% \\

\cmidrule{1-6}

\multirow{6}{*}{ORD}
& wisconsin & (32, 193, 5)  & \textbf{1.24} &  1.26 & 1.26 \\
& stocks & (9, 949, 5) & 0.167 & 0.168 & \textbf{0.160} \\
% & boston & (13, 505, 5) & \textbf{0.285} & 0.355 & \textbf{0.285} \\
& machine & (6, 208, 10)  & 0.634 & 0.628  & \textbf{0.628} \\
& abalone & (10, 4176, 10) & 0.520 & 0.526 & 0.520 \\
& auto & (7, 391, 10) & 0.589 & 0.621 & \textbf{0.585} \\
\cmidrule{1-6}

&   & $(d, n, M)$ &  &  &  \\ 
\midrule

\multirow{1}{*}{SEQ}
& ocr & (128, 6877, 26) & 16.2\% & 16.3\% & 16.2\% \\
% & pos & (13, 43681, 45) & - & - & - \\
% & ner & (20, 21001, 9) & - & - & - \\
% & conll & (19, 10948, 23) & - & - & - \\
\cmidrule{1-6}

\multirow{5}{*}{RNK}
& glass & (9, 214, 6) & 17.7\%  & - & \textbf{17.4\%} \\
% & pendigits & (16, 10992, 10) & - & b & c \\
% & wisconsin & (16, 194, 16) & - & - & c \\
% & segment & (18, 2310, 7) & a & b &  c \\
% & calhousing & (4, 20640, 4) & a & b & c \\
% & cpu-small & (6, 8192, 5) & a & b & c \\
& bodyfat & (7, 252, 7) & 79.6\% & - & 79.6\% \\
% & authorship & (70, 841, 4) & \textbf{9.83E-02} & - & 9.92E-02 \\
% & iris & (4, 150, 3) & a & b & c \\
& wine & (13, 178, 3) & 5.06\% & - & \textbf{4.34\%} \\
& vowel & (10, 528, 11) & 33.7\% & - &\textbf{ 32.2\%} \\
& vehicle & (18, 846, 4) & \textbf{14.8\%} & - & 15.0\% \\
% & fried & (9, 40768, 5) & a & b & c \\

\bottomrule
\end{tabular}
    }
    \caption{Average test losses on the 14 splits for multi-class classification (first), ordinal regression (second), sequence prediction (third) and ranking (forth). We show in percentage the losses for multi-class, sequence prediction and ranking since they are between zero and one. We show in bold the lowest test loss between the direct classifiers $\operatorname{M^3N}$ and $\operatorname{M^4N}$.}
    \label{table:experiments}
\end{table}

\begin{table}[ht!] \label{table:experiments_oracles}
    \hspace{-0.15cm}
    \resizebox{!}{0.14\linewidth}{
    \begin{tabular}{@{}llllll@{}}
%%%% MULTILABEL
\toprule

Dataset
& W-S  & $K=10$ & $K=30$ & $K=50$ & $K=100$ \\ 
\midrule
% \multirow{2}{*}{wisconsin}
% & yes & 0.9827 / 0.2522 & 0.9827 / 0.1819 & 0.9655 / 0.08854 &  0.9827 / dual 0.05099 & 0.94827 / 0.03284 \\
% & no & 1.2241 / 1.7273 & 1.2068 / 1.1391 & 1.1379 / 0.6942 & 1.0862 / 0.2956 & 1.0517 / 0.1597 \\
% \cmidrule{1-7}
\multirow{2}{*}{machine}
& yes  & 0.42 / 0.57 & 0.41 / 0.43 & 0.41 / 0.22 & 0.41 / 0.13\\
& no  & 0.50 / 4.41 & 0.50 / 2.74 & 0.44 / 1.25 & 0.42 / 0.63 \\
\cmidrule{1-6}
\multirow{2}{*}{auto}
& yes  & 0.56 / 1.55 & 0.55 / 1.29 & 0.51 / 0.81 & 0.50 / 0.44 \\
& no  & 0.61 / 2.66 & 0.57 / 1.79 & 0.53 / 0.89 & 0.51 / 0.47 \\
\bottomrule
\end{tabular}
    }
    \caption{We show the (final ordinal test loss / average oracle error at the last epoch) for $\operatorname{M^4N}$s trained with $100$ passes on data with different iterations of \cref{alg:spmp} with and without warm-starting. }
    \vspace{-0.3cm}
\end{table}
\section{Conclusion} \label{sec:conclusion}

In this paper, we introduced max-min margin Markov networks ($\operatorname{M^4N}$s), a method for general structured prediction, that has the same algorithmic and theoretical properties as the regular binary SVM, that is, quantitative convergence bounds through a linear comparison inequality, as well as efficient optimization algorithms. Our experiments show its performance on classical structured prediction problems when using RKHS hypothesis spaces. It would be interesting to extend the analysis of the proposed algorithm by rigorously proving the linear dependence in the number of samples when using the warm-start strategy and incorporating a line-search strategy.
 
 \section*{Acknowledgements} The authors would like to thank Mathieu Blondel, Martin Arjovsky and Simon Lacoste-Julien for useful discussions. This work was funded in part by the French government under management of Agence Nationale de la Recherche as part of the ``Investissements d'avenir'' program, reference ANR-19-P3IA-0001 (PRAIRIE 3IA Institute). We also acknowledge support the European Research Council (grant SEQUOIA 724063). ANV received support from ''La Caixa'' Foundation.

\bibliography{references}
\bibliographystyle{icml2020}

\newpage

\appendix \label{sec:appendix}
\onecolumn

{\Huge\upshape\bfseries\mathversion{bold} Organization of the Appendix }

\begin{itemize}
\item [\textbf{A.}] {\large\bfseries Notation}
\item [\textbf{B.}] {\large\bfseries Geometrical Properties}
\begin{itemize}
    \item [\textbf{B.1.}] {\bfseries Geometry of the loss L} 
    \item [\textbf{B.2.}] {\bfseries Geometry of the loss S} 
    \item [\textbf{B.3.}] {\bfseries Relation between cell complexes} 
    \item [\textbf{B.4.}] {\bfseries Examples} 
\end{itemize}
\item [\textbf{C.}] {\large\bfseries Theoretical properties of the Surrogate}
\begin{itemize}
    \item [\textbf{C.1.}] {\bfseries Fisher consistency} \\
    Here {\bf Theorem~3.2} is derived as {\bf Theorem C.3}.
    \item [\textbf{C.2.}] {\bfseries Comparison inequality and calibration function} 
    \item [\textbf{C.3.}] {\bfseries Characterizing the calibration function for Max-Min Markov Networks} 
    \item [\textbf{C.4.}] {\bfseries Quantitative lower bound} \\
    Here {\bf Theorem~3.3} is derived as {\bf Theorem C.8}.
    \item [\textbf{C.5.}] {\bfseries Computation of the constant for specific losses} 
\end{itemize}
\item [\textbf{D.}] {\large\bfseries Sharp Generalization Bounds for Regularized Objectives} \\
Here {\bf Theorem~3.4} is derived.
\item [\textbf{E.}] {\large\bfseries Max-min margin and dual formulation}
\begin{itemize}
    \item [\textbf{E.1}] {\bfseries Derivation of the Dual Formulation}
    \item [\textbf{E.2}] {\bfseries Computation of the Dual Gap}
\end{itemize}
\item [\textbf{F.}] {\large\bfseries Generalized Block-Coordinate Frank-Wolfe}
\begin{itemize}
    \item [\textbf{F.1.}]{\bfseries General Convergence Result}
    \item [\textbf{F.2.}]{\bfseries Application to $\boldsymbol{\operatorname{M^4N}}$}\\
    Here {\bf Theorem 5.1} is proven using the analysis from section F.1.
\end{itemize}
\item [\textbf{G.}] {\large\bfseries Solving the Oracle with Saddle Point Mirror Prox}
 \begin{itemize}
    \item [\textbf{G.1.}]{\bfseries Saddle Point Mirror Prox (SP-MP}
    \item [\textbf{G.2.}]{\bfseries Max-Min Oracle for Sequences}
    \item [\textbf{G.3.}]{\bfseries Max-Min Oracle for Ranking and Matching}
\end{itemize}
\item [\textbf{H.}] {\large\bfseries Generalization Bounds for $\boldsymbol{\operatorname{M^4N}}$ solved via GBCFW and Approximate Oracle}\\
Here {\bf Theorem 5.2} is proven combining the results from section D and Theorem 5.1.
\end{itemize}

\section{Notation}
In this section we introduce some notation that will be useful in the rest of the appendix.

\textbf{Notation on the structured prediction setting.}
Denote by $\calP(A)$ the set of subsets of the set $A$.
We define the following quantities
\begin{itemize}
    \item[-] \textit{Marginal polytope}: $\calM = \operatorname{hull}(\phi(\calY)) = \{v\in\Rspace{k}~|~v = \sum_{y\in\calY}\alpha_y\phi(y), \alpha\in\Delta_{\calY}\}$.
    \item[-] \textit{Normal cone of $\calM$ at $\mu$}: $\calN_{\calM}(\mu) = \{u\in\Rspace{k}~|~\langle \mu' - \mu, u\rangle\leq 0, \forall\mu'\in\calM\}$.
    \item[-] \textit{Conditional moments}: $\nu(q) = \Expect_{y'\sim q}\phi(y')\in\calM$ where $q\in\operatorname{Prob}(\calY)$.
    \item[-] \textit{Conditional risk}: $\ell(y, \mu)\defeq \Expect_{y'\sim q}L(y, y') = \phi(y)^\top A\mu$.
    \item[-] \textit{Bayes risk}: $\ell(\mu) \defeq \min_{y\in\calY}\ell(y, \mu)$
    \item[-] \textit{Minus Bayes risk}: $\Omega(\mu) \defeq -\ell(\mu) + 1_{\calM}(\mu)$.
    \item[-]\textit{Excess conditional risk}: $\delta\ell(y, \mu) = \ell(y, \mu) - \ell(\mu)$
    \item[-] \textit{Optimal predictor set}: $y^\star(\mu) = \argmin_{y\in\calY}~\ell(y, \mu)\subseteq\calY$.
    \item[-] \textit{Marginal polytope cell complex}: $\calC(\calM) = (y^\star)^{-1}\circ y^\star(\calM)\subset \calP(\calM)$.
\end{itemize}

\textbf{Notation on the max-min surrogate.}
\begin{itemize}
    \item[-] \textit{Partition function}: $\Omega^*(v) = \max_{\mu\in\calM}\ell(\mu) + v^\top \mu$.
    \item[-] \textit{Surrogate loss}: $S(v,y) = \Omega^*(v) - v^\top \phi(y)$.
    \item[-] \textit{Conditional surrogate risk}: $s(v,\mu)\defeq \Expect_{y\sim q}S(v, y) = \Omega^*(v) - v^\top \mu$.
    \item[-] \textit{Bayes surrogate risk}: $s(\mu) \defeq \min_{v\in\Rspace{k}}s(v, \mu) \quad (=\ell(\mu))$.
    \item[-] \textit{Excess surrogate conditional risk}: $\delta s(v, \mu) \defeq s(v, \mu) - s(\mu)$.
    \item[-] \textit{Optimal predictors}: $v^\star(\mu) = \argmin_{v\in \Rspace{k}}~s(v, \mu)\subset\Rspace{k}$.
    \item[-] \textit{Surrogate space cell complex}: $\calC(\Rspace{k}) = v^\star(\calM)\subset \calP(\Rspace{k})$.
\end{itemize}

\section{Geometrical Properties}

In this section, we study the rich geometrical properties of the max-min surrogate construction. The geometrical interpretation provides a valuable intuition on different key mathematical objects appearing in the further analysis needed for the proofs of the main theorems. More precisely, we show that the max-min surrogate loss $S$ defines a partition~$\calC(\Rspace{k})\subset \calP(\Rspace{k})$ of the surrogate space $\Rspace{k}$ which is dual to the partition~$\calC(\calM)\subset \calP(\calM)$ of the marginal polytope~$\calM$ defined by $L$. Moreover, we show that the mapping between those partitions is the subgradient mapping $\partial\Omega$ with inverse $\partial\Omega^*$ (see \cref{fig:appgeometry}). Visualization for binary 0-1 loss, multi-class 0-1 loss, absolute loss for ordinal regression and Hamming loss are provided in \cref{sec:examples}.

\begin{figure}[ht!]
    \centering
    \includegraphics[width=0.7\textwidth]{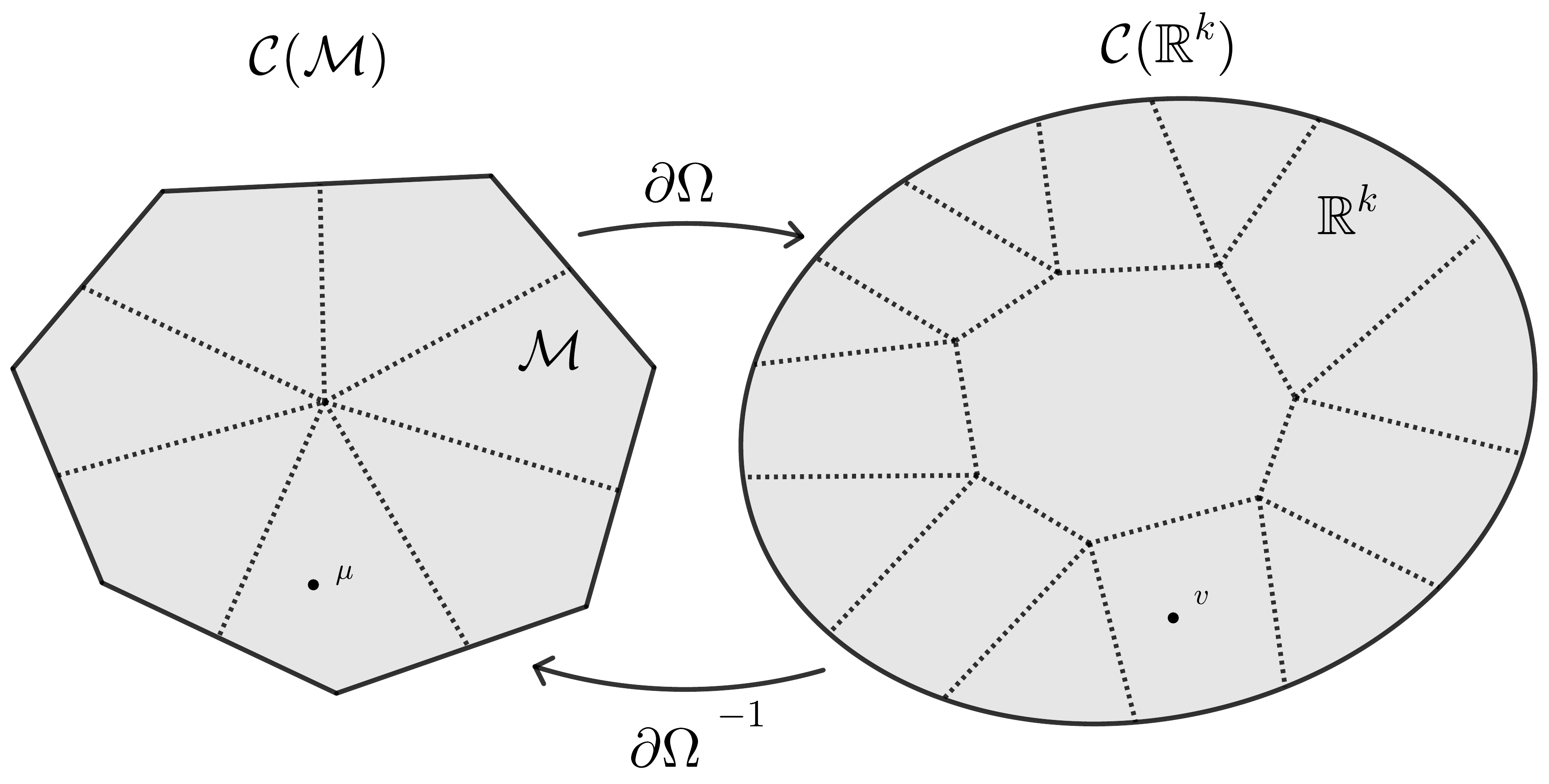}
    \caption{The cell complex $\calC(\calM)$ on the marginal polytope $\calM$ maps to the cell complex $\calC(\Rspace{k})$ on the surrogate space $\Rspace{k}$ through the subgradient mapping of the partition function $\partial\Omega$.}
    \label{fig:appgeometry}
\end{figure}

Following \cite{finocchiaro2019embedding}, we now introduce the definition of a \emph{cell complex}.
\begin{definition}[Cell complex] \label{def:cellcomplex} A cell complex in $\Rspace{k}$ is a set $\calC$ of faces (of dimension $0,\ldots,k$) such that:
\begin{itemize}
    \item[(i)] union to $\Rspace{k}$.
    \item[(ii)] have pairwise disjoint relative interiors.
    \item[(iii)] any nonempty intersection of faces $F,F'$ in $\calC$ is a face of $F$ and $F'$ and an element of $\calC$.
\end{itemize}
\end{definition}

Any convex affine-by-parts function has an associated cell complex defined by considering the polytope corresponding to the epigraph of the function and projecting the faces down to to the domain. Moreover, if $f(v)$ is convex affine-by-parts, the cell complex associated to $f(v)$ and $f(v) + v^\top a$ are the same for any $a$.

\subsection{Geometry of the Loss L}

The convex affine-by-parts function $\Omega(\mu) = -\ell(\mu) + 1_{\calM}(\mu)$ naturally defines a cell complex of its domain $\calM$ (see for instance \cite{ramaswamy2016convex, nowak2019general}). This can be constructed by considering the polyhedra corresponding to the epigraph of $-\ell(\mu)$ and then projecting the faces to $\calM$. Each face corresponds to a different group of active hyperplanes in the definition of $\ell(\mu)$.  The cell complex can be defined as $\calC(\calM) = \{\overline{(y^\star)^{-1}\circ y^\star(\mu)}~|~\mu\in\calM\}\subset\calP(\calM)$, i.e., each face is defined as the set of moments that share the same set of optimal predictors. Note that $\calC(\calM)$ contains faces of 0-dimensions (points) up to faces of $k$-dimensions.

\subsection{Geometry of the Loss S}

Recall that $S(v,y) = \Omega^*(v) - v^\top \mu$ and $\Omega(\mu) = -\ell(\mu) + 1_{\calM}(\mu)$, where $\ell(\mu)$ is concave affine-by-parts. In particular, as $\Omega$ is convex affine-by-parts with compact domain, then $\Omega^*$ is convex affine-by-parts with full-dimensional domain $\Rspace{k}$. The projection of the faces of the convex polyhedron defined as the epigraph of $\Omega^*$ defines a cell complex $\calC(\Rspace{k})\subset\calP(\Rspace{k})$ in the (unbounded) vector space $\Rspace{k}$. The cell complex defined by $\Omega^*(v)$ is the same as the one defined by $s(v,\mu) = \Omega^*(v) - v^\top \mu$ for every $\mu\in\calM$. The faces of $\calC(\Rspace{k})$ can be written as $v^\star(\mu) = \argmin_{v\in \Rspace{k}}~s(v, \mu)\in \calP(\Rspace{k})$ for a certain $\mu\in\calM$, i.e., the faces are the minimizers of the conditional surrogate risk. Hence, we can write in a compact form $v^\star(\mu)\in\calC(\Rspace{k})$.

\subsection{Relation between Cell Complexes}
Recall that $\calC(\calM)$ is generated by projecting the faces of the epigraph of $\Omega$ while $\calC(\Rspace{k})$ is generated by projecting the faces of the epigraph of $\Omega^*$. 
The subgradients are well-defined in the cell complexes and define a bijection between them:
\begin{equation*}
    \partial\Omega:\calC(\calM)\rightarrow\calC(\Rspace{k}),\hspace{1cm}
    \partial\Omega^*:\calC(\Rspace{k})\rightarrow\calC(\calM),\hspace{1cm} \partial\Omega^* = (\partial\Omega)^{-1}.
\end{equation*}
Moreover, we have that 
\begin{align*}
    \operatorname{dim}(\partial\Omega(F)) &= \operatorname{dim}(\calM) - \operatorname{dim}(F), \hspace{1cm} &\forall F\in \calC(\calM) \\
    \operatorname{dim}(\partial\Omega^*(F')) &= \operatorname{dim}(\calM) - \operatorname{dim}(F'), \hspace{1cm} &\forall F'\in \calC(\Rspace{k}), \\
\end{align*}
where $F,F'$ are faces of $\calC(\calM), \calC(\Rspace{k})$, respectively.
\subsection{Examples}\label{sec:examples} Let's now provide some concrete examples of cell complexes and the associated mapping subgradient mapping for several classical tasks.

\paragraph{Binary Classification.} The output space is $\calY = \{-1, 1\}$. The loss is $L(y, y') = 1(y\neq y')$ with affine decomposition~$a=1$, $\phi(1) = (1, 0)^\top $, $\phi(-1) = (0, 1)^\top $ and $A = -I_{2\times 2}$. The marginal polytope is $\calM=\Delta_2$.

\begin{equation*}
    \Omega(p) = -\min(p, 1-p) + 1_{\Delta_2}(p), \hspace{0.5cm}
    \Omega^*(v) = \max(|v|, 1/2),  \hspace{0.5cm} S(v,y) = \Omega^*(v) - yv.
\end{equation*}

\begin{figure}[ht!]
    \centering
    \includegraphics[width=0.7\textwidth]{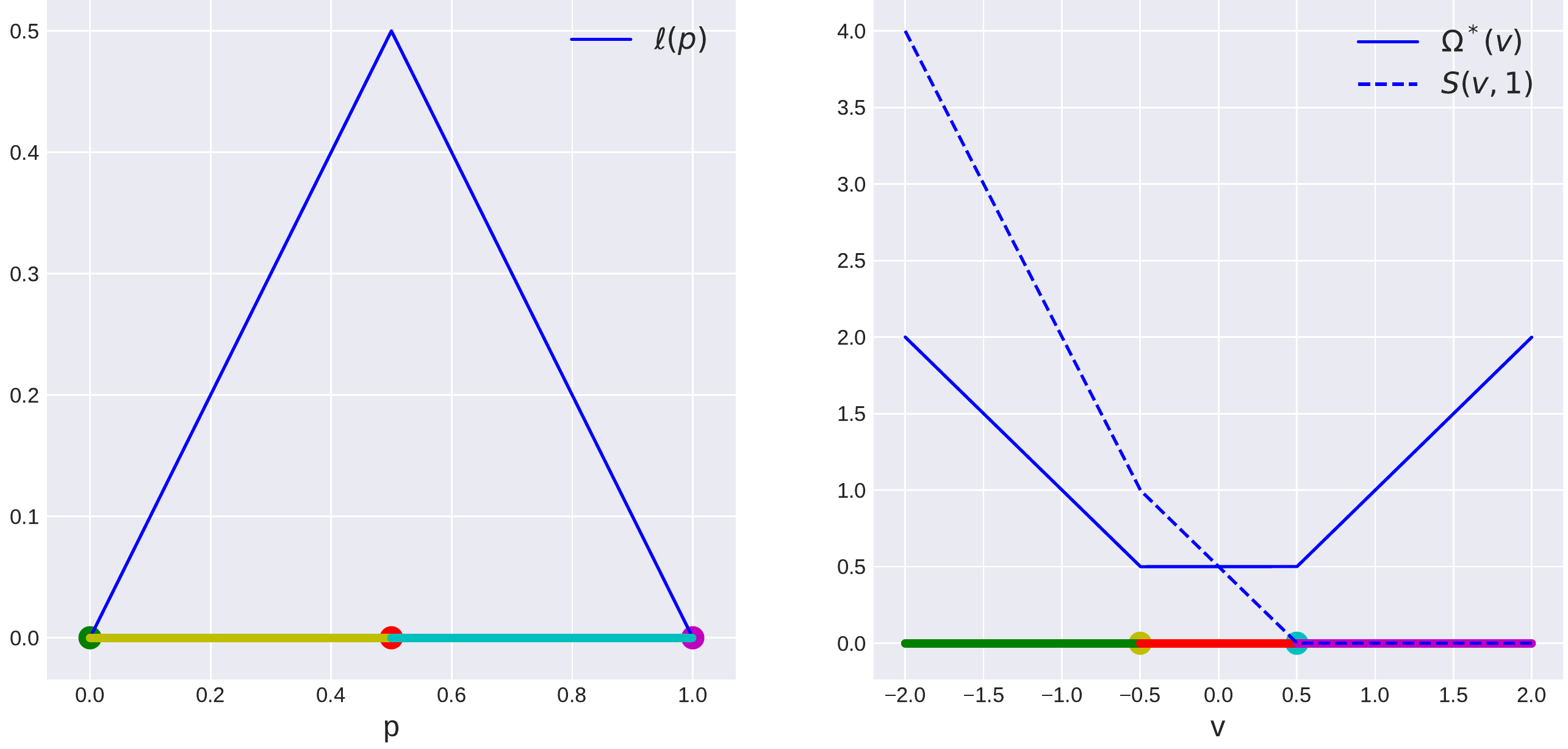}
    \caption{Binary 0-1 loss. \textbf{Left:} The Bayes risk $-\Omega$ is a concave polyhedral function defined in $\Delta_2=[0,1]$. The faces of the induced cell-complex are the 0-dimensional faces $\{0\}, \{1/2\}, \{1\}$ and the 1-dimensional faces $[0, 1/2]$, $[1/2, 1]$. \textbf{Right:} The partition function is a convex polyhedral function defined in $\Rspace{}$. The faces of the induced cell-complex are the 0-dimensional faces~$\{-1/2\}$, $\{1/2\}$ and the 1-dimensional faces $(-\infty, -1/2]$, $[1/2, 1/2]$, $[1/2, +\infty)$. }
    \label{fig:appbinary}
\end{figure}

See \cref{fig:appbinary}. The mapping between cells is

\begin{equation*}
    \begin{array}{lll}
       & \partial\Omega(\{0\}) = (-\infty, -1/2], \hspace{0.5cm} & \partial\Omega^*((-\infty, -1/2]) = \{0\}  \\
        & \partial\Omega(\{1/2\}) = [-1/2,1/2], \hspace{0.5cm} & \partial\Omega^*([-1/2,1/2]) = \{1/2\}  \\
     & \partial\Omega(\{1\}) = [1/2, +\infty), \hspace{0.5cm} & \partial\Omega^*([1/2, +\infty)) = \{1\}  \\
        & \partial\Omega([0, 1/2]) = \{-1/2\}, \hspace{0.5cm} & \partial\Omega^*(\{-1/2\}) = [0, 1/2]  \\
        & \partial\Omega([1/2, 1]) = \{1/2\}, \hspace{0.5cm} & \partial\Omega^*(\{1/2\}) = [1/2, 1]  \\
    \end{array}
\end{equation*}

\paragraph{0-1 Multi-class Classification.} The output space is $\calY = [k] = \{1, \cdots, k\}$. The loss is $L(y, y') = 1(y\neq y')$ with affine decomposition $a=1$, $\phi(y) = e_y$ and $A = -I_{k\times k}$, where $e_y\in\Rspace{k}$ is the $y$-th vector of the canonical basis in $\Rspace{k}$. The marginal polytope is the $k$-dimensional simplex $\calM=\Delta_k$.

In this case, $\ell(p) = 1 - \|p\|_{\infty}$ and so
\begin{equation*}
    \Omega(p) = \|p\|_{\infty} - 1 + 1_{\Delta_k}(p), \hspace{0.5cm}
    \Omega^*(v) = 1 + \max_{j\in[k]}\bigg\{\frac{1}{j}\sum_{r=1}^j v_{(r)} -\frac{1}{j}\bigg\},
\end{equation*}
where $v_{(1)}\geq\cdots\geq v_{(k)}$ (see \cref{fig:appmulticlass}).

\begin{figure}[ht!]
    \centering
    \includegraphics[width=0.7\textwidth]{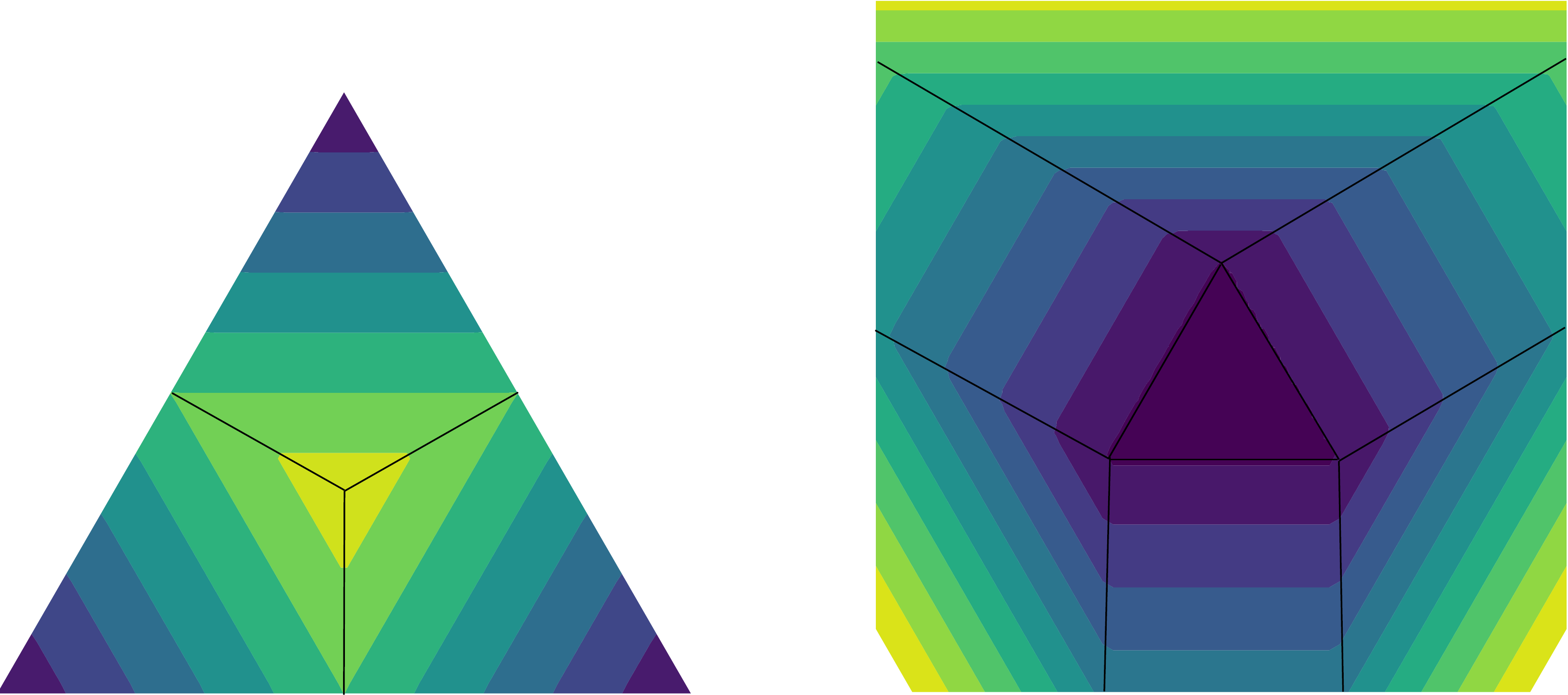}
    \caption{Multi-class 0-1 loss ($k=3$). \textbf{Left:} The Bayes risk $-\Omega$ has a pyramid shape centered at the simplex. The cell-complex $\calC(\Delta_k)$ has $2^k$ 0-dimensional faces (points) and $k$ full-dimensional faces. In the figure, the set of points are the center point, the 3 vertices of the triangle and the 3 middle points in the triangle face. The 3 full-dimensional faces are the 3 colored zones.  \textbf{Right:} The partition function~$\Omega^*$ is a convex polyhedral function. The cell-complex has $k$ 0-dimensional faces (points) and $2^k$ full-dimensional faces. In the figure, the set of points are the 3 vertices of the triangle in the center and the full-dimensional faces are the colored zones.}
    \label{fig:appmulticlass}
\end{figure}

Note that the subgradient mapping $\partial\Omega$ sends the $2^k$ 0-dimensional faces (points) and the full-dimensional faces of $\calC(\Delta_k)$ to the full-dimensional faces and 0-dimensional faces of $\calC(\Rspace{k-1})$, respectively.

\paragraph{Ordinal Regression.}  The output space is the same as for multiclass classification, but in this case there is an implicit ordering between outputs: $1\prec 2\prec \cdots\prec k$. This is encoded using the absolute difference loss $L(y, y') = |y-y'|$. We consider the affine decomposition $\phi(y) = e_y\in\Rspace{k}$, $A = (|i-j|)_{i,j\in[k]^2}$ and $a=0$. It is possible to obtain a closed form expression for the partition function (see Thm. 6 by \cite{fathony2018consistent}):
\begin{equation*}
    \Omega^*(v) = \frac{1}{2}\max_{i,j\in[k]}~v_i + v_j + j - i.
\end{equation*}

In \cref{fig:appordinal} we plot the Bayes risk and the partition function for the ordinal loss. Note that that the topology of the cell-complex is different from the previous example.

\begin{figure}[ht!]
    \centering
    \includegraphics[width=0.7\textwidth]{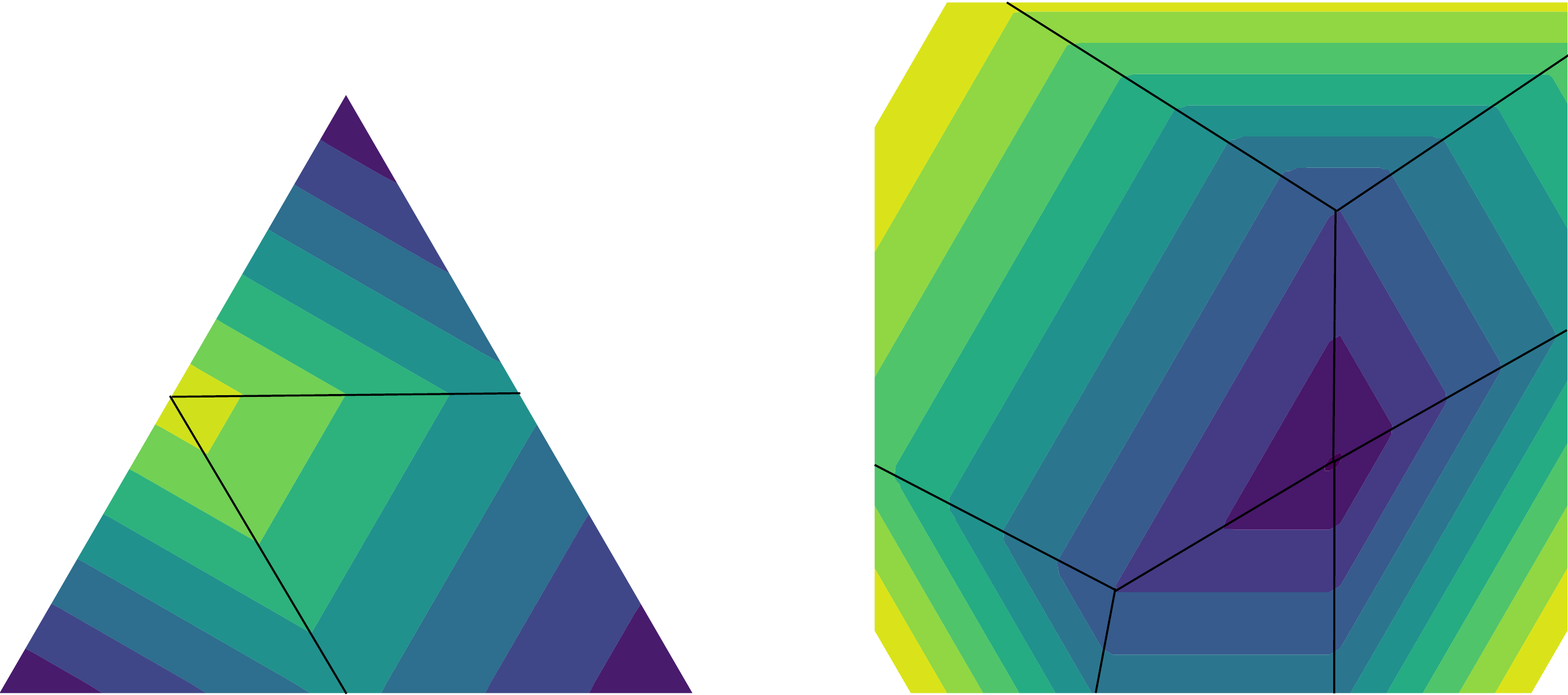}
    \caption{Absolute difference loss ($k=3$). \textbf{Left:} The Bayes risk $-\Omega$ has an asymmetrical pyramid shape with the tip in one face of the simplex. \textbf{Right:} The partition function $\Omega^*$ has a different topology than the one from multi-class.}
    \label{fig:appordinal}
\end{figure}

\paragraph{Multi-label Classification with Hamming Loss.} This corresponds to \cref{ex:factorgraph} with unary potentials. Let $\calY=\Pi_{m=1}^M\calY_m$ with $\calY_m=\{1,\dots, R\}$. 
We consider the Hamming loss defined as an average of multi-class losses: $L(y,y') = \frac{1}{M}\sum_{m=1}^M1(y_m\neq y_m')$. The marginal polytope factorizes as $\calM = \Pi_{m=1}^M\Delta_R$. The Bayes risk decomposes additively as the sum of the Bayes risks of the individual multi-class losses and the partition function decomposes analogously. In \cref{fig:apphamming} we plot the Bayes risk and the partition function for $R=M=2$.

\begin{figure}[ht!]
    \centering
    \includegraphics[width=0.7\textwidth]{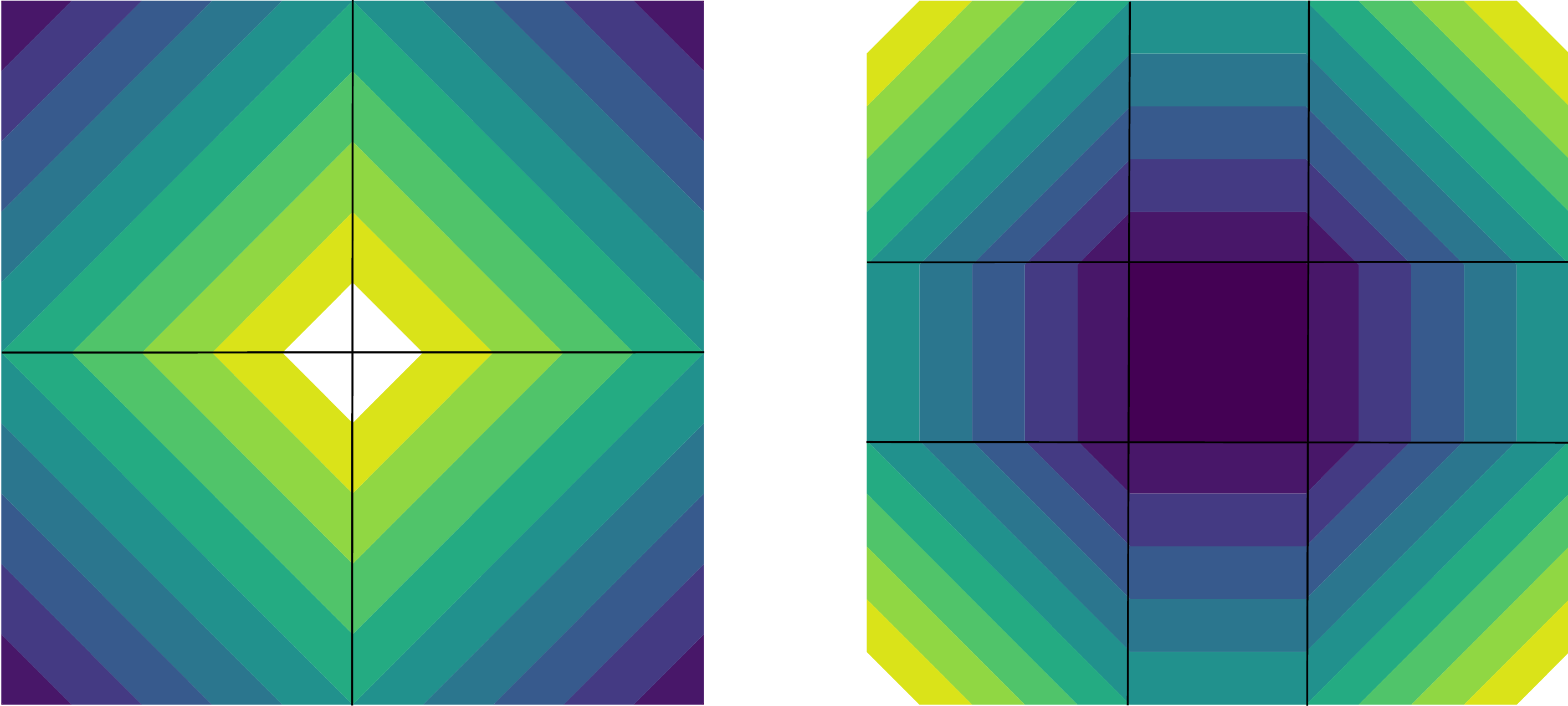}
    \caption{Hamming loss for $(R=M=2)$. \textbf{Left:} The marginal polytope is the cube $\calM=[0,1]^2$ and the Bayes risk $-\Omega$ has a pyramid shape centered in the cube. \textbf{Right:} The partition function $\Omega^*$.}
    \label{fig:apphamming}
\end{figure}

\section{Theoretical Properties of the Surrogate}\label{app:theorysurrogate}

The goal of this section is to prove the two theoretical requirements for the surrogate method. These are \emph{Fisher consistency~(1)} and a \emph{comparison inequality~(2)}:
\begin{equation*}
    \begin{array}{ll}
    \text{(1)} & \calE(f^\star) = \calE(d\circ g^\star)  \\
    \text{(2)} & \zeta(\calE(d\circ g) - \calE(f^\star)) \leq \calR(g) - \calR(g^\star).
\end{array}
\end{equation*}
for all measurable $g:\calX\xrightarrow{}\calH$, where $\zeta:\Rspace{}\xrightarrow{}\Rspace{}$ is such that $\zeta(\varepsilon)\to 0$ if $\varepsilon\to 0$. Fisher consistency ensures that the optimum of the surrogate loss $g^\star$ provides the Bayes optimum $f^\star$ of the problem. However, in practice the optimum of the surrogate is never attained and so one wants to control how close $f = d\circ g$ is to $f^\star$ in terms of the estimation error of~$g$ to~$g^\star$. 
The comparison inequality gives this quantification by relating the excess expected risk $\calE(d\circ g) - \calE(f^\star)$ to the excess expected surrogate risk $\calR(g) - \calR(g^\star)$, which allows to translate rates from the surrogate problem to the original problem.

Let's start first by showing that $s(\mu) = \ell(\mu)$ for all $\mu\in\calM$, i.e., that the minimizers of the conditional surrogate risks coincide.
\begin{lemma}\label{lem:bayesrisks}
    The Bayes risk and the surrogate Bayes risk are the same:
    \begin{equation*}
        s(\mu) = \min_{v\in\Rspace{k}}~s(v,\mu) = \min_{y'\in\calY}~\ell(y, \mu) = \ell(\mu), \hspace{1cm} \forall \mu\in\calM.
    \end{equation*}
\end{lemma}
\begin{proof}
Note that $s(\mu) = \min_{v\in\Rspace{k}}~s(v,\mu) = \min_{v\in\Rspace{k}}~\Omega^*(v) - v^\top\mu = -\Omega = \ell(\mu) - 1_{\calM}(\mu)$.
\end{proof}
Note that this is not the case for smooth surrogates. It was noted by \cite{finocchiaro2019embedding} (see Prop. 1 and 2) that consistent polyhedral surrogates necessarily satisfy the property of matching Bayes risks.

\subsection{Fisher Consistency} \label{app:fisherconsistency}
The following \cref{th:minimzersurrogate} characterizes the form of the exact minimizer of the conditional surrogate risk $s(v, \mu)$. 
\begin{proposition} \label{th:minimzersurrogate}
    Let $\mu\in\calM$ and $y^\star(\mu) = \argmin_{y\in\calY}\phi(y)^\top A\mu$ be the set of optimal predictors. Then, we have that
    \begin{equation}
        \operatorname{hull}(-A^\top \phi(y))_{y\in y^\star(\mu)} + \calN_{\calM}(\mu) = \argmin_{v\in \Rspace{k}} s(v, \mu).
    \end{equation}
\end{proposition}
\begin{proof}
The proof consists in noticing that $\operatorname{hull}(-A^\top \phi(y))_{y\in y^\star(\mu)}$ is a subgradient at $\mu$ of the non-smooth convex function~$\Omega = -\ell(\mu) + 1_{\calM}(\mu)$ with compact domain $\calM$. That is,
\begin{align*}
    \Omega(\mu) &= -\min_{y'\in\calY}\phi(y')^\top A\mu + 1_{\calM}(\mu) \\
    &= -\phi(y)^\top A\mu + 1_{\calM}(\mu), \hspace{1cm} y \in y^\star(\mu).
\end{align*}
The subgradient reads $\partial \Omega(\mu) = -\operatorname{hull}(A^\top \phi(y))_{y\in y^\star(\mu)} + \calN_\calM(\mu)$, where $ \calN_\calM(\mu)$ is the normal cone of $\calM$ at the point~$\mu$. Then, using Fenchel duality we have that
\begin{align*}
   \partial \Omega(\mu) = \argmin_{v\in \Rspace{k}}~\Omega^*(v) - v^\top \mu
   = \argmin_{v\in \Rspace{k}} s(v, \mu).
\end{align*}
\end{proof}

Let $\rho(\cdot|x)$ be the conditional distribution of outputs and $\mu(x)= \Expect_{y\sim \rho(\cdot|x)}\phi(y)$. If we assume that the set of points $x\in\calX$ for which $\calN_\calM(\nu(x))\neq \{0\}$ has measure zero, then we have that $g^\star(x) \in -\operatorname{hull}(A^\top \phi(y))_{y\in y^\star(\rho(\cdot|x))}$ almost-surely. Thus, we can write $g^\star(x) = -\sum_{y\in y^\star(\rho(\cdot|x))}\alpha_yA^\top\phi(y)$ with $\sum_{y\in y^\star(\rho(\cdot|x))}\alpha_y=1$. We have Fisher consistency as 
\begin{equation*}
    f^\star(x)\in\argmax_{y'\in\calY}~\phi(y')^\top g^\star(x) = \argmin_{y'\in\calY}\sum_{y\in y^\star(\rho(\cdot|x))}\alpha_yL(y', y).
\end{equation*}

\subsection{Comparison Inequality and Calibration Function} \label{app:comparisoninequality}

The goal of this section is to explicitly compute a comparison inequality. We will show that the relation between both excess risks is linear and that the constants appearing scale nicely with the natural dimension of the structured problem and not with the total number of possible outputs $|\calY|$ which can potentially be exponential.

The main object of study will be the so-called \emph{calibration function}, which is defined as the `worst' comparison inequality between both excess conditional risks.
\begin{definition}[Calibration function \cite{osokin2017structured}]\label[definition]{def:calibrationfunction}
    The calibration function $\zeta:\Rspace{}_{+}\longrightarrow\Rspace{}_{+}$ is defined for~$\varepsilon\geq 0$ as the infimum of the excess conditional surrogate risk when the conditional risk is at least $\varepsilon$:
    \begin{equation*}\label{eq:optimalcalibration}
        \zeta(\varepsilon) = \inf\delta s( v, \mu)\quad \text{such that} \quad
        \delta\ell(d\circ v, \mu)
        \geq\varepsilon,~\mu\in\calM,~ v\in \Rspace{k}.
\end{equation*}
We set $\zeta(\varepsilon)=\infty$ when the feasible set is empty.
\end{definition}
Note that $\zeta$ is non-decreasing in $[0, +\infty)$, not necessarily convex 
(see Example 5 by \cite{bartlett2006convexity})
and also~$\zeta(0)=0$. Note that a larger $\zeta$ is better because we want a large $\delta s( v,\mu)$ to incur small $\delta\ell(d\circ v,\mu)$. The following \cref{th:calibrationrisks} justifies \cref{def:calibrationfunction}.
\begin{theorem}[Comparison inequality in terms of calibration function \cite{osokin2017structured}]\label[theorem]{th:calibrationrisks} Let $\bar{\zeta}$ be a convex lower bound of $\zeta$.
We have 
\begin{equation}\label{eq:riskcalibration}
\bar{\zeta}(\ERL(d\circ g) - \ERL(\fstar))\leq \ERS(g) - \ERS(\gstar)
\end{equation}
for all $ g:\calX\rightarrow \Rspace{k}$.
The tightest convex lower bound $\bar{\zeta}$ of $\zeta$ is its lower convex envelope which is defined by the Fenchel bi-conjugate $\zeta^{**}$.
\end{theorem}
\begin{proof} ~Note that by the definition of the calibration function, we have that
\begin{equation}\label{th:calibrationbayes}
    \zeta(\delta\ell(d\circ  g( x),\mu(x))) \leq \delta s( g( x),\mu(x)),
\end{equation}
where $\mu(x) = \Expect_{y'\sim\rho(\cdot|x)}\phi(y')$. The comparison between risks is then a consequence of Jensen's inequality:
\begin{align*}
    \bar{\zeta}(\ERL(d\circ  g) - \ERL(\fstar))
   &= \bar{\zeta}(\Expect_{ x\sim\rho_{\calX}} \delta\ell(d\circ  g( x),\mu(x))) && \\
   &\leq \Expect_{ x\sim\rho_{\calX}} \bar{\zeta}(\delta\ell(d\circ  g(x),\mu(x))) && (\text{Jensen ineq.})\\
   &\leq \Expect_{\sim\rho_{\calX}} \zeta(\delta\ell(d\circ  g( x),\mu(x))) && (\bar{\zeta}\leq\zeta) \\
   &\leq \Expect_{ x\sim\rho_{\calX}} \delta s( g( x),\mu(x)) && \\
   & = \ERS( g) - \ERS(\gstar). &&
\end{align*}
\end{proof}

\subsection{Characterizing the Calibration Function for Max-Min Margin Markov Networks}\label{app:calibrationfunction-characterization}

Following \cite{osokin2017structured}, we write the calibration function in terms of pairwise interactions.

\begin{lemma}[Lemma 10]\label{lem:calibrationfunction-split}
    We can re-write the calibration function $\zeta(\varepsilon)$ as
    \begin{equation*}
     \zeta(\varepsilon) = \min_{ y\neq y'}~\zeta_{ y, y'}(\varepsilon),   
    \end{equation*}
     where
     \begin{equation}\label{eq:zetayy}
         \zeta_{ y, y'}(\varepsilon) = \left\{\begin{array}{llr}
             \min_{v,\mu\in \Rspace{k}} & \delta s( v, \mu)  &  \\
             \text{s.t} & \delta\ell( y', \mu) \geq  \varepsilon & (\varepsilon-\text{suboptimality}) \\
             & y = y^\star(\mu) & (\text{optimal prediction}) \\
             & y' = d\circ v & (\text{prediction}) \\
             & \mu\in\calM &  \\
         \end{array}\right.
     \end{equation}
\end{lemma}

\begin{proof}
The idea of the proof is to decompose the feasibility set of the optimization problem into a union of sets enumerated by the pairs $(y, y')$ corresponding to the optimal prediction $y$ and the prediction $y'$. Let's first define the sets $V(y)\subset\Rspace{k}$ and~$\calM_{y,y',\varepsilon}\subset\calM$.
\begin{enumerate}
    \item Define the prediction sets as $V(y) \defeq \{v\in \Rspace{k}~|~v^\top (\phi(y) - \phi(y'))> 0, \forall y'\in\calY\}\subset \Rspace{k}$ to denote the set of elements in the surrogate space $\Rspace{k}$ for which the prediction is the output element $y\in\calY$.  Note that the sets $V(y)$ do not contain their boundary, but their closure can be expressed as
    \begin{equation*}
    \overline{ V}(y')\defeq \{v\in \Rspace{k}~|~v^\top (\phi(y') - \phi(y))\geq 0, \forall y\in\calY\}.    
    \end{equation*}
    Note that $\bigcup_{y'\in\calY}\overline{V}(y') =  \Rspace{k}$.
    \item If $v\in V(y')$, the feasible set of conditional moments $\mu$ for which output $y$ is one of the best possible predictions (i.e., $\ell(y',\mu)-\ell(y,\mu)\geq\varepsilon$) is
\begin{equation*}
    \calM_{y,y',\varepsilon} = \{\mu\in\calM~|~\ell(y,\mu)=\ell(\mu)~|~\ell(y',\mu)-\ell(y,\mu)\geq\varepsilon\}.
\end{equation*}
\end{enumerate}

The union of the sets $\{\overline{ V}(y')\times\calM_{y,y',\varepsilon}\}_{y,y'\in\calY}$ exactly equals the feasibility set of the optimization problem \cref{eq:optimalcalibration}. We can then re-write the calibration function as
\begin{equation}\label{eq:calibrationfunctionfactorized}
    \zeta(\varepsilon) = \min_{y\neq y'}\left\{\begin{array}{ll}
        \min_{v,\mu} &\delta s(v, \mu)  \\
        \text{s.t.} & v\in V(y') \\
         & \mu\in\calM_{y,y',\varepsilon}
    \end{array}\right. .
\end{equation}

Finally, by Lemma 27 of \cite{zhang2004statistical}, the function $\delta s(v, \mu)$ is continuous w.r.t both $\mu$ and $v$, allowing to substitute the sets $V(y')$ in \cref{eq:calibrationfunctionfactorized} by their closures $\overline{ V}(y')$ without changing the minimum.
\end{proof}

Until now, the results were general for any calibration function. We will now construct a lower bound on the calibration function for $\operatorname{M^4N}$s. Let's first introduce some notation.
\paragraph{Notation.}
\begin{itemize}
    \item[-] Let $\calM_0$ be the finite set of $0$-dimensional faces (points) of the cell complex $\calC(\calM)\subset\calP(\calM)$, or equivalently (mapped by $\partial\Omega$), the full dimensional faces of the cell complex~$\calC(\Rspace{k})\subset\calP(\Rspace{k})$. Note that $|\calM_0|$ is finite.
    \item[-] Let $w(y) = -A^\top \phi(y)$.
\end{itemize}
Recall that in \cref{lem:calibrationfunction-split} we split the optimization problem into $|\calY|(|\calY|-1)$ optimization problems corresponding to all possible (ordered) pairs of different optimal prediction and prediction. The following \cref{thm:characterization-calibration} further splits the inner optimization problems into some faces of the cell complex $\calC(\Rspace{k})$ and simplifies the objective function into an affine function.
\begin{theorem}[Calibration Function for the surrogate loss of $\operatorname{M^4N}$]\label{thm:characterization-calibration}
    We have that 
    \begin{equation*}
        \zeta_{y,y'}(\varepsilon) = \min_{\bar{\mu}\in\calM_0(y,y')}\zeta_{y,y',\bar{\mu}}(\varepsilon),
    \end{equation*}
    where  $\calM_0(y,y') = \{\bar{\mu}\in\calM_0~|~w(y), w(y')\in\partial\Omega(\bar{\mu})\}\subseteq\calM_0\subset\calM$, and
    \begin{equation}\label{eq:zetayymu}
     \zeta_{y,y', \bar{\mu}}(\varepsilon) = \left\{\begin{array}{lllr}
        \underset{ v,\mu\in\Rspace{k}}{\min} & \langle w(y) -  v, \mu - \bar{\mu}\rangle & & \\
         \text{s.t} & \langle w(y) - w(y'), \mu\rangle \geq \varepsilon & & (\varepsilon-\text{suboptimality}) \\
         & \langle w(y) - w(z), \mu\rangle \geq 0, & \forall z\in \calY & (\text{optimal prediction}) \\
         & y' = d\circ v &  & (\text{prediction}) \\
         & v\in\partial\Omega(\bar{\mu}) & & (\text{face in }\calC(\Rspace{k})) \\
         & \mu\in\calM &  \\
     \end{array}\right.
    \end{equation}
\end{theorem}
\begin{proof} We split the proof into three steps. First, we split the optimization problem w.r.t $v\in\Rspace{k}$ among the faces $\partial\Omega(\calM_0)$ of the complex cell $\calC(\Rspace{k})\subset\calP(\Rspace{k})$. Second, we show that the minimizer is achieved in a face $\partial\Omega(\bar{\mu})$ such that $w(y),w(y')\in\partial\Omega(\bar{\mu})$ and simplify the objective function. Finally, we update the notation of some constraints. 

\textbf{1st step. Split the optimization problem according to the affine parts.} Recall that $s(v,\mu)$ is defined as a supremum of affine functions, where each affine function corresponds to a $\bar{\mu}\in\calM_0$: $s(v,\mu) = \sup_{\bar{\mu}\in\calM_0}~\ell(\bar{\mu}) + v^\top(\bar{\mu}-\mu)$. Using that~$\bigcup_{\bar{\mu}\in\calM_0}\overline{\partial\Omega(\bar{\mu})} = \Rspace{k}$ and the continuity of the loss, we split problem \eqref{eq:zetayy} into $|\calM_0|$ minimization problems and define 
\begin{equation*}
\zeta_{y,y'}(\varepsilon) = \min_{\bar{\mu}\in\calM_0}\zeta_{y,y',\bar{\mu}}(\varepsilon),    
\end{equation*}
where $\zeta_{y,y',\bar{\mu}}(\varepsilon)$ is given by problem \eqref{eq:zetayy} with the additional constraint $v\in\partial\Omega(\bar{\mu})$.

\begin{figure}[h!]
    \centering
    \includegraphics[width=0.5\textwidth]{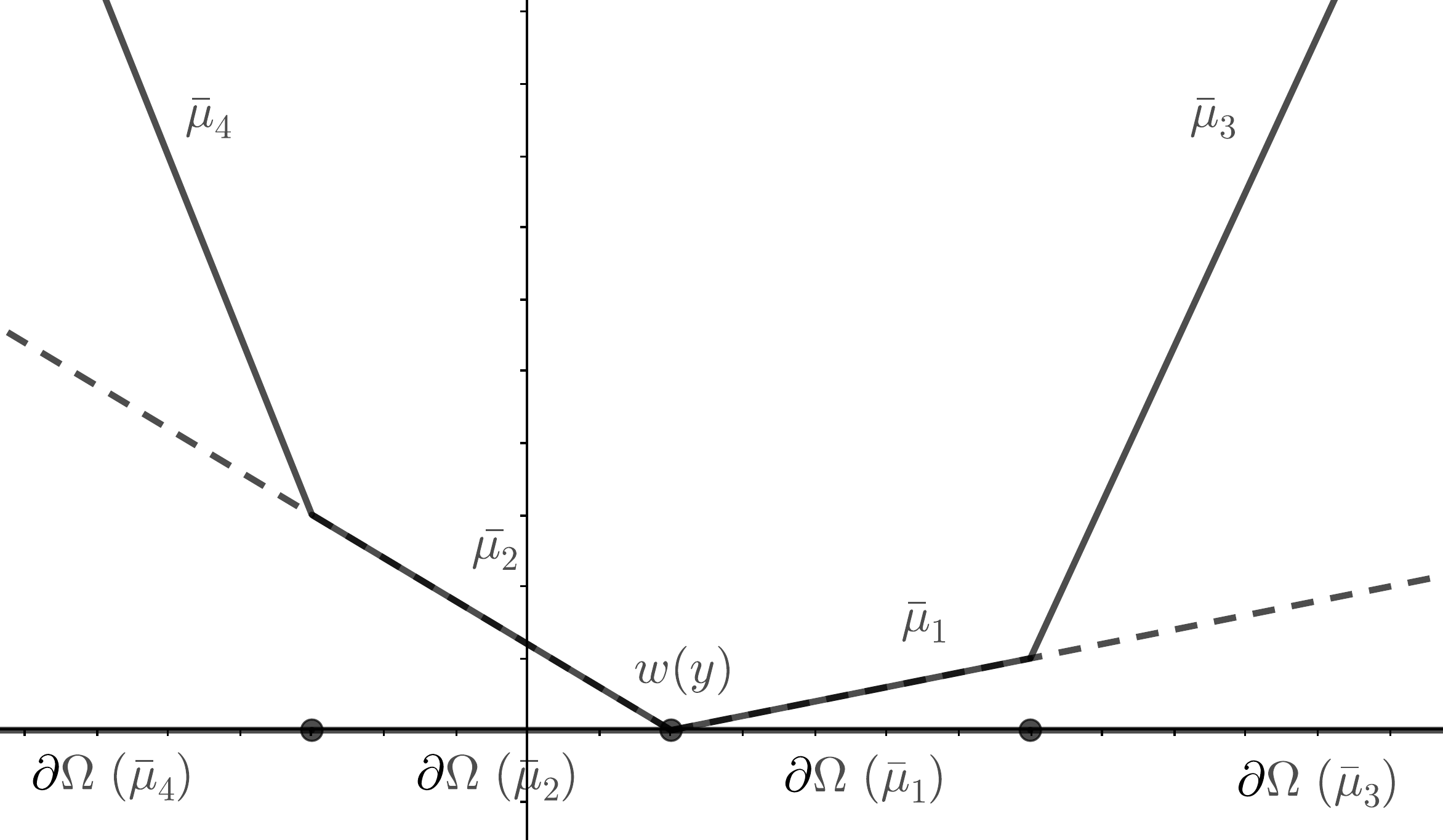}
    \caption{The minimizer of the objective $\delta s(v, \mu)$ over $\Rspace{k}$ is $w(y)$. In order to minimize the objective under the constraints, we only need to consider the faces which include the minimizer $w(y)$. In the figure above, we can safely remove the optimization over the faces~$\partial\Omega(\bar{\mu}_4)$ and $\partial\Omega(\bar{\mu}_3)$. }
    \label{fig:lowerboundproof}
\end{figure}

\textbf{2nd step. Reduce the number of considered affine parts and simplify objective.} We will show that 
\begin{equation*}
    \min_{\bar{\mu}\in\calM_0}\zeta_{y,y',\bar{\mu}}(\varepsilon) = \min_{\bar{\mu}\in\calM_0(y,y')}\zeta_{y,y',\bar{\mu}}(\varepsilon),
\end{equation*}
where $\calM_0(y,y') = \{\bar{\mu}\in\calM_0~|~w(y), w(y')\in\partial\Omega(\bar{\mu})\}\subset\calM_0$. Moreover, when $\bar{\mu}\in\calM_0(y,y')$, the objective function in the definition of $\zeta_{y,y',\bar{\mu}}(\varepsilon)$ takes the affine form $\langle w(y) -  v, \mu - \bar{\mu}\rangle$.
In order to see this, let's make the following observations.
\begin{itemize}
    \item[-] We have that $w(y')$ must belong to the feasibility set as $y'=d\circ w(y')$. And so, we must have $w(y')\in\partial\Omega(\bar{\mu})$.
    \item[-] Fix $\mu$ in the feasibility set of \cref{eq:zetayy}. As $y$ is the optimal prediction, $w(y)$ is a minimizer of the conditional surrogate risk: $\min_{v'}~s(v',\mu) = s(w(y), \mu)$. The objective function $s(v,\mu) - s(w(y), \mu)\geq 0$ is a convex affine-by-parts function with minimizer $w(y)$. We can lower bound this quantity by simply considering the affine parts $\bar{\mu}\in\calM_0$ that include the minimizer, i.e., $w(y)\in\partial\Omega(\bar{\mu})$ (see \cref{fig:lowerboundproof}). Moreover, note that if $w(y)\in\partial\Omega(\bar{\mu})$, then $\Omega^*( v) = \Omega^*(   w(y)) + \langle \bar{\mu}, v - w(y)\rangle$, as $\bar{\mu}$ is the slope of the affine part $\partial\Omega(\bar{\mu})$. Using that $\min_{v'\in \Rspace{k}}s(v',\mu) = s(w(y), \mu) = \Omega^*(  w(y)) - \langle   w(y), \mu\rangle$, we have that 
    \begin{align*}
        \delta s(v, \mu) &= s(v, \mu) - s(w(y), \mu) \\
        &= \Omega^*( v) - \Omega^*(   w(y)) + \langle    w(y) -  v, \mu\rangle \\
        &= \langle w(y) -  v, \mu - \bar{\mu}\rangle.
    \end{align*}
\end{itemize}
\textbf{3rd step. Re-write constraints in terms of $\boldsymbol{w(y)}$.} The constraint $ y =  y^\star(\mu) = \argmin_{y\in\calY} \phi( y)^\top A\mu$ is equivalent to~$\ell(z, \mu) - \ell(y, \mu)\geq 0$ for all $z\in\calY$, which can be written $\delta\ell(z, \mu) =\langle  w(y) -   w(z), \mu\rangle$, for all $z\in\calY$. Similarly, the constraint  $\delta\ell(y', \mu) \geq \varepsilon$ reads $\langle  w(y) -   w(y'), \mu\rangle \geq \varepsilon$.
\end{proof}

In order to state \cref{th:maintheoremcalibrationfunction}, let's first define the function $\lambda_{y'}^{\mu}:\partial\Omega(\mu)\rightarrow\Rspace{}_{\geq 0}$. By \cref{th:minimzersurrogate}, we know that~$\partial\Omega(\mu) = \operatorname{hull}(w(y))_{y\in y^\star(\mu)} + \calN_{\calM}(\mu)$.
In general, there exist multiple ways to describe a vector $v\in\partial\Omega(\mu)$ as~$v = \sum_{y\in y^\star(\mu)} \lambda_y w(y) + n$ with $\lambda\in\Delta_{\calY}$ and $n\in\calN_{\calM}(\mu)$. The function $v\mapsto \lambda_{y'}^{\mu}(v)$ is defined as the maximal weight of the vector $w(y')$ over all possible decompositions:
\begin{equation}\label{eq:lambdadefinition}
    \lambda_{y'}^{\mu}(v) = \left\{\begin{array}{llr}
        \underset{\lambda, n}{\max} & \lambda_{ y'} \\
        \text{s.t} &v = \sum_{y\in y^\star(\mu)} \lambda_y w(y) + n \\
        & \lambda\in y^\star(\mu) \\
        & n\in\calN_{\calM}(\mu)
    \end{array}\right. .
\end{equation}
The following \cref{th:maintheoremcalibrationfunction} gives a \emph{constant positive} lower bound of the ratio $\zeta(\varepsilon)/\varepsilon$ as a minimization of $\lambda_{y'}^{\bar{\mu}}(v)$ over the prediction set of $y'$.
\begin{theorem}\label{th:maintheoremcalibrationfunction} We have that
\begin{equation*}
    \zeta(\varepsilon) = \min_{y'\in\calY}\min_{\bar{\mu}\in\calM_0(y')}~\zeta_{y',\bar{\mu}}(\varepsilon)
\end{equation*},
where $\calM_0(y') = \{\bar{\mu}\in\calM_0~|~w(y')\in\partial\Omega(\bar{\mu})\}\subseteq\calM_0\subset\calM$ and 
\begin{equation}\label{eq:calibrationfunction-compressedform}
    \zeta_{y',\bar{\mu}}(\varepsilon) / \varepsilon \geq  \left\{\begin{array}{ll}
         \underset{v\in\partial\Omega(\bar{\mu})}{\min} & \lambda_{y'}^{\bar{\mu}}(v) \\
         \text{s.t} &  y'=d\circ v
     \end{array}\right. .
\end{equation}
\end{theorem}
\begin{proof}
We split the proof into four steps. First, we remove some constraints and write the optimization problem in terms of~$\mu - \bar{\mu}$. Second, we construct the dual of the linear program associated to the minimization w.r.t. $\mu$ and extract the variable $\varepsilon$ as a multiplying factor in the objective, thus showing the linearity of the calibration function. Then, we add a simplex constraint to simplify the problem and finally, we put everything together to obtain the desired result.

\textbf{1st step. Write optimization w.r.t $\boldsymbol{\mu}$ in terms of $\boldsymbol\mu-\boldsymbol{\bar{\mu}}$ by removing some constraints.} 
Let's proceed with the following editions of the constraints of \eqref{eq:zetayymu} to obtain a lower bound:
\begin{enumerate}
    \item The cone $\calN_{\calM}(\bar{\mu})$ is polyhedral, as it is the normal cone of the convex polytope $\calM$ at the point $\bar{\mu}$. Hence, it is a finitely generated cone \cite{de2012algebraic}, which can be described as
    \begin{equation*}
        \calN_{\calM}(\bar{\mu}) = \{a_1n_1 + \cdots + a_rn_r~|~a_i\geq 0, n_i\in\Rspace{k}\}.
    \end{equation*}
    Let's now replace the constraint $\mu\in\calM$ (last constraint of \eqref{eq:zetayymu}) by the constraints $\langle -n_i, \mu - \bar{\mu} \rangle\geq 0$ where $1\leq i \leq r$ and $n_i$ are the generators of the cone $\calN_{\calM}(\bar{\mu})$.
    \item Note that by construction, we have that 
    \begin{equation}\label{eq:optmubar}
        \langle\phi(z), A\bar{\mu}\rangle = \langle\phi(z'), A\bar{\mu}\rangle, \quad \forall z,z'\in y^\star(\bar{\mu}),
    \end{equation}
    as $z,z'$ are optimal for the conditional moments $\bar{\mu}$.
    Let's remove from the second line of constraints of \eqref{eq:zetayymu} the ones corresponding to $z\in\calY\setminus y^\star(\bar{\mu})$ and use \eqref{eq:optmubar} for the remaining constraints.
    We obtain
    \begin{equation*}
     \zeta_{y,y', \bar{\mu}}(\varepsilon) \geq \left\{\begin{array}{lll}
        \underset{ v,\mu\in\Rspace{k}}{\min} & \langle w(y) -  v, \mu - \bar{\mu}\rangle & \\
        \text{s.t} & \langle A^\top (\phi(y')-\phi(y)), \mu - \bar{\mu}\rangle \geq \varepsilon, & \\
        & \langle A^\top (\phi(z)-\phi(y)), \mu - \bar{\mu}\rangle \geq 0, & \forall z\in y^\star(\bar{\mu}) \\
        & \langle -n_i, \mu - \bar{\mu} \rangle\geq 0, & 1\leq i \leq r \\
         & \langle v, \phi( y') - \phi(z)\rangle \geq 0, & \forall z\in\calY \\
     \end{array}\right.
    \end{equation*}
    \item Do the change of variables $\mu'=\mu-\bar{\mu}$ and re-define $\mu\defeq \mu'$ to ease notation.
    \begin{equation*}
     \zeta_{y,y', \bar{\mu}}(\varepsilon) \geq \left\{\begin{array}{lll}
        \underset{ v,\mu\in\Rspace{k}}{\min} & \langle w(y) -  v, \mu\rangle & \\
        \text{s.t} & \langle A^\top (\phi(y')-\phi(y)), \mu\rangle \geq \varepsilon, & \\
        & \langle A^\top (\phi(z)-\phi(y)), \mu\rangle \geq 0, & \forall z\in y^\star(\bar{\mu}) \\
        & \langle -n_i, \mu\rangle\geq 0, & 1\leq i \leq r \\
         & \langle v, \phi( y') - \phi(z)\rangle \geq 0, & \forall z\in\calY \\
     \end{array}\right.
    \end{equation*}
\end{enumerate}

\textbf{2nd step. Linearity in $\boldsymbol{\varepsilon}$ via duality.} 
Define $\overline{\calY} = y^\star(\bar{\mu})$. Let's now study separately the linear program corresponding to the variables $\mu$, which reads as
\begin{equation*}
   \textbf{(P)} \hspace{1cm} \left\{\begin{array}{lll}
    \underset{\mu}{\min} & \langle w(y) -  v, \mu\rangle & \\
    \text{s.t} & \langle A^\top (\phi(y')-\phi(y)), \mu\rangle \geq \varepsilon, & \\
    & \langle A^\top (\phi(z)-\phi(y)), \mu\rangle \geq 0, & \forall z\in \overline{\calY} \\
    & \langle -n_i, \mu\rangle\geq 0, & 1\leq i \leq r \\
 \end{array}\right.
\end{equation*}
Let's consider the dual formulation $\textbf{(D)}$ of $\textbf{(P)}$: 
\begin{equation*}
    \textbf{(D)} \hspace{1cm} \left\{\begin{array}{llr}
        \underset{\lambda\in\Rspace{\overline{\calY}+r}}{\max} & \varepsilon\lambda_{ y'} \\
        \text{s.t} & A^\top \sum_{z\in \overline{\calY}}
         \lambda_{z}(\phi(y) - \phi(z)) + \sum_{i=1}^r\lambda_i^nn_i = A^\top \phi(y) + v \\
        & \lambda_y \geq 0 & y\in\overline{\calY} \\
        & \lambda_i^n\geq 0 & 1 \leq i \leq r
    \end{array}\right. ,
\end{equation*}
where we have used that $w(y) = -A^\top \phi(y)$. 

\textbf{3rd step. Simplify by adding a simplex constraint (dependence of optimal prediction $\boldsymbol{y}$ disappears).} As problem $\textbf{(D)}$ is written as a maximization, we can lower bound the objective by adding constraints.
If we add the constraint $\sum_{z\in\overline{\calY}}\lambda_z=1$, the term $A^\top\phi(y)$ simplifies and we obtain the following lower bound
\begin{equation}\label{eq:addsimplexconstraint}
    \left\{\begin{array}{llr}
        \underset{\lambda\in\Rspace{\overline{\calY}+r}}{\max} & \varepsilon\lambda_{ y'} \\
        \text{s.t} & A^\top \sum_{z\in \overline{\calY}}
         \lambda_{z}w(z) + \sum_{i=1}^r\lambda_i^nn_i = v \\
        & \lambda_y \geq 0 & y\in\overline{\calY} \\
        &\sum_{z\in\overline{\calY}}\lambda_z=1 \\
        & \lambda_i^n\geq 0 & 1 \leq i \leq r
    \end{array}\right. ,
\end{equation}
Note that the term $\sum_{i=1}^r\lambda_i^nn_i$ with $\lambda_i^n\geq 0$ covers all possible normal cone vectors, and so the maximization can be written over vectors in $\calN_{\calM}(\bar{\mu})$. Hence, \cref{eq:addsimplexconstraint} can be written as 
\begin{equation}\label{eq:addsimplexconstraint2}
    \left\{\begin{array}{llr}
        \underset{\lambda, n}{\max} & \varepsilon\lambda_{ y'} \\
        \text{s.t} &v = \sum_{y\in \overline{\calY}} \lambda_y w(y) + n \\
        & \lambda\in\Delta_{\overline{\calY}} \\
        & n\in\calN_{\calM}(\bar{\mu})
    \end{array}\right. .
\end{equation}

\textbf{4th step. Putting everything together.} Recall that problem \eqref{eq:addsimplexconstraint2} is a function of $v$. The desired lower bound is constructed by minimizing the quantity \eqref{eq:addsimplexconstraint2} under the constraints $v\in\partial\Omega(\bar{\mu})$ and $y'=d\circ v$.
\end{proof}
\subsection{Quantitative Lower Bound. } \label{app:quantitativaelowerbound}
The compressed form of the calibration function \eqref{eq:calibrationfunction-compressedform} given by \cref{th:maintheoremcalibrationfunction} is still far from a quantitative understanding on the value of the function. The following \cref{th:calibrationfunction-quantitative} provides a quantitative lower bound under mild assumptions on the loss $L$. 

\paragraph{Assumption on L.} $L$ is symmetric and there exists $C > 0$ such that
\begin{equation}\label{eq:assumptiononL}
    y \in \argmin_{y' \in \calY} \mathbb{E}_{z \sim \alpha} L(y',z) \implies \alpha_y \geq 1/C > 0,
\end{equation}
for all $\alpha \in \Delta_{\calY}$.

\begin{theorem}\label{th:calibrationfunction-quantitative}
Assume \eqref{eq:assumptiononL}. Then, for any $\varepsilon>0$, the calibration function is lower bounded by 
\begin{equation*}
    \zeta(\varepsilon) \geq \frac{\varepsilon}{D},
\end{equation*}
where $D=\max_{y'\in\calY}D_{y'}$ and 
\begin{equation}\label{eq:definitionDy}
    1/D_{y'} = \min_{\bar{\mu}\in\calM_0(y')}\left\{\begin{array}{ll}
         \underset{\alpha\in\Delta_{\overline{\calY}}}{\min}\underset{\beta\in\Delta_{\overline{\calY}}}{\max} & \beta_y \\
         \text{s.t} &  A^\top\Expect_{z\sim\alpha}\phi(z) = A^\top\Expect_{z'\sim\beta}\phi(z') \\
         & y'\in \argmin_{y \in \calY} \mathbb{E}_{z \sim \alpha}~L(y,z)
     \end{array}\right. ,
\end{equation}
where $\overline{\calY} = y^\star(\bar{\mu})$.
\end{theorem}
\begin{proof}
We use the notation $\nu(\alpha) = \Expect_{y\sim\alpha}\phi(y)$ for $\alpha\in\Delta_{\calY}$. Let's first show that 
\begin{equation}\label{eq:calibrationfunction-firststep}
    \zeta_{y', y,\bar{\mu}}(\varepsilon) / \varepsilon \geq  \left\{\begin{array}{ll}
         \underset{\alpha\in\Delta_{\overline{\calY}}}{\min} & \lambda_{y'}^{\bar{\mu}}(-A^\top \nu(\alpha)) \\
         \text{s.t} &  y'\in \argmin_{y \in \calY} \mathbb{E}_{z \sim \alpha}~L(y,z)
     \end{array}\right. .
\end{equation}
In order to see this, note that as $v\in\partial\Omega(\bar{\mu})$, we can write $v=-A^\top\nu(\alpha) + n_v$ where $\alpha\in\Delta_{\overline{\calY}}$ and $n_v\in\calN_{\calM}(\bar{\mu})$. We will show that condition $y'=d\circ v$ implies
\begin{equation*}
    \langle \phi(z) - \phi(y'), A^\top \nu(\alpha)\rangle \geq 0, \hspace{0.5cm} \forall z\in \calY.
\end{equation*}
The condition $y'=d\circ v$ is equivalent to $\langle \phi(z) - \phi(y'), A^\top \nu(\alpha) - n_v\rangle \geq 0$ for all $z\in\calY$. By definition of $\calN_{\calM}(\bar{\mu})$, we have that $n_v$ satisfies $\langle s - \bar{\mu}, n_v \rangle \leq 0$ for any $s \in \calM$. Now let $z \in \calY\setminus\{y'\}$ and consider the representation~$\bar{\mu} = c_{y',\bar{\mu}} \phi(y') + c_{z,\bar{\mu}}\phi(z) + (1-c_{y',\bar{\mu}}-c_{z,\bar{\mu}}) r$ with $r \in \hull(\calY\setminus\{y',z\})$ and $0 \leq c_{y',\bar{\mu}}, c_{z,\bar{\mu}}\leq 1$. Since $n_v \in \calN_{\calM}(\bar{\mu})$ satisfies~$\langle s - \bar{\mu}, n_v \rangle \leq 0$ also for $s = (c_{y',\bar{\mu}} + c_{z,\bar{\mu}}) \phi(y') + (1-c_{y',\bar{\mu}} - c_{z,\bar{\mu}})r~ \in ~\calM$, when $c_{z,\bar{\mu}} > 0$ we have
$$ 0 \geq c_{z,\bar{\mu}}^{-1} \langle s - \bar{\mu}, n_v \rangle = \langle \phi(y') - \phi(z), n_v \rangle.$$
From \eqref{eq:assumptiononL}, we know that $c_{z,\bar{\mu}} \geq 1/C > 0$ for all $z \in y^\star(\bar{\mu})$. Then we have 
$$\langle \phi(z) - \phi(y'), n_v \rangle \geq 0, \quad \forall~ z \in y^\star(\bar{\mu}).$$
Note moreover that $y^\star(\nu(\alpha)) \subseteq y^\star(\bar{\mu})$. Indeed, by the assumption \eqref{eq:assumptiononL}, we have that $\alpha_z = 0$ implies $z \notin y^\star(\nu(\alpha))$ and since $\alpha \in \Delta_{y^\star(\bar{\mu})}$ we have that $\alpha_z = 0$ for $z \notin y^\star(\bar{\mu})$. Since $y' \in y^\star(\bar{\mu})$ by construction of $\mu$ and $A$ is symmetric due to the symmetry of $L$ 
$$\mathbb{E}_{t \in \alpha} L(t,z) - \mathbb{E}_{t \in \alpha} L(t,y') = \langle \phi(z) - \phi(y'), A\nu(\alpha)\rangle = \langle \phi(z) - \phi(y'), A^\top \nu(\alpha)\rangle \geq \langle \phi(z) - \phi(y'), n_v \rangle \geq 0,$$
for all $z \in y^\star(\bar{\mu})$. Hence, \cref{eq:calibrationfunction-firststep} is proven. Finally, setting $n=0$ in the definition \eqref{eq:lambdadefinition} of $\lambda_{y'}^{\bar{\mu}}(-A^\top\nu(\alpha))$, we obtain the desired lower bound.
\end{proof}

\begin{corollary}\label{cor:constantC}
Under the same assumptions of \cref{th:calibrationfunction-quantitative}, we have that $C\geq D$, and so
\begin{equation*}
    \zeta(\varepsilon) \geq \frac{\varepsilon}{C}.
\end{equation*}
\end{corollary}
\begin{proof}
This can be seen by setting $\alpha=\beta$ in \cref{eq:definitionDy}.
\end{proof}

\paragraph{Exponential constants in the calibration function.} We argue that the constant $D$ from \cref{th:calibrationfunction-quantitative} does not grow as the size of the output space $\calY$ when the problem is structured, i.e., $k\ll|\calY|$. On the other hand, the constant $C$ from Assumption \eqref{eq:assumptiononL} and \cref{cor:constantC} can take exponentially large values (of the order of $|\calY|$) when the problem is structured. We show this by studying the calibration function for \cref{ex:factorgraph} and \cref{ex:matching} in the next section.

\subsection{Computation of the Constant for Specific Losses}
\label{sec:constantC}

\paragraph{Calibration function for factor graphs (\cref{ex:factorgraph}).} Assume that we only have unary potentials and the individual losses are the 0-1 loss, which means that $A=-Id$ is the negative identity. Assume also that each part takes binary values, i.e.,~$R=2$. The constant $C$ can be as large as $|\calY|=2^M$, by considering the uniform distribution $\alpha_y=1/2^M$ for all $y$, which is optimal for every output. On the other hand, as the marginals for the uniform distribution are $(1/2, 1/2)$ for every $m$, one can take $\beta = 1/2\delta_y + 1/2\delta_{-y}$, and so $D$ is $2$.

\paragraph{Calibration function for ranking and matching (\cref{ex:matching}).} In this case, the constant $C$ can be as large as $|\calY| = M!$ by considering the uniform distribution $\alpha_y=1/M!$ for all $y$. This corresponds to $\Expect_{z\sim\alpha}\phi(z) = 1/M11^\top$. For this distribution, the value of the constant $D$ is $M$, because one can write $1/M11^\top$ as the the uniform distribution over $M$ different permutations.

\begin{proposition}[Lipschitz Multi-class] 
Let $\calY = \{1,\dots,k\}$,$k \geq 2$ and assume that $L$ be symmetric. If there exists~$q \in [0,1)$ such that for all $y,z \in \calY$
$$\sum_{t \in \calY \setminus \{y,z\}} |L(t,y) - L(t,z)| ~~\leq~~ q ~L(y,z),$$
then the calibration function for $M^4N$ is bounded by 
$$ \zeta(\varepsilon) \geq H \epsilon, \quad H \geq \frac{1 - q}{k-q} > 0.$$
\end{proposition}
\begin{proof}
First we prove that $L$ satisfies \eqref{eq:assumptiononL}. Then we apply Theorem~\ref{thm:characterization-calibration}. Let $\alpha \in \Delta_\calY$ and assume that $y \in \argmin_{y \in \calY} \mathbb{E}_{t \sim \alpha} L(y,t)$. This is equivalent to the following
$$ \sum_{t \in \calY} \alpha_t L(t, y) \leq \sum_{t \in \calY} \alpha_t L(t, z), \quad \forall z \in \calY.$$
In particular fix as $z = \argmax_{t \in \calY\setminus\{y\}} \alpha_t$. By symmetry of the loss, the equation above is equivalent to
$$ (\alpha_z - \alpha_y) L(y, z) \leq \sum_{t \in \calY\setminus\{y,z\}}  \alpha_t (L(t,z) - L(t,y)).$$
Let $s = \arg\max_{t \in \calY\setminus\{y,z\}} \alpha_t$, then
$$\sum_{t \in \calY\setminus\{y,z\}}  \alpha_t (L(t,z) - L(t,y)) \leq (\max_{t \in \calY\setminus\{y,z\}} \alpha_t) \sum_{t \in \calY\setminus\{y,z\}} |L(t,z) - L(t,y)| \leq \alpha_s ~ q ~L(y,z).$$
Note that by construction $\alpha_s \leq \alpha_z$, so 
$$(\alpha_z - \alpha_y) L(y, z) \leq \alpha_s ~ q ~L(y,z) \leq \alpha_z ~ q ~L(y,z),$$
from which we have $\alpha_z (1-q) \leq \alpha_y$.
Since $\alpha_z$ is the maximum probability over $\calY \setminus \{y\}$, then it can not be smaller than~$(1 - \alpha_y)/(k-1)$, so $\alpha_z \geq (1 - \alpha_y)/(k-1)$. From which we derive
$$ \alpha_y \geq \frac{1-q}{k-q}.$$
This holds for any $\alpha \in \Delta_\calY$, $y \in \calY$ and implies that  \eqref{eq:assumptiononL} is valid for $L$, with $c \geq \frac{1-q}{k-q}$. Then we can apply Theorem~\ref{thm:characterization-calibration} obtaining the desired result.
\end{proof}

\begin{proposition}[Decomposable Multi-label Loss]\label{prop:decomposablelosses}
Let $\calY = \calY_1\times\cdots\times\calY_M$, $L(y,y') = \sum_{m=1}^ML_m(y_m, y_m')$ and~$\phi(y) = (e_{y_m})_{m\in M}$. Let $\zeta_m$ be the calibration function of $L_m$ and assume $\zeta_m(\varepsilon) \geq \varepsilon/C_m$, with $C_m > 0$. The calibration function $\zeta$ associated to $L(y,y')$ has the following form:
\begin{equation*}
     \zeta(\varepsilon) ~~\geq~~ \varepsilon / (\max_{m\in[M]}C_m).
\end{equation*}
\end{proposition}
\begin{proof}
We have $\delta \ell(y, \mu) = \sum_{m=1}^M\delta\ell_m(y_m, \mu_m)$ and the surrogate conditional loss decomposes additively as 
\begin{equation*}
    s(v, \mu) = \sum_{m=1}^Ms_m(v_m, \mu_m), \hspace{0.5cm}s_m(v_m, \mu_m) = \max_{q\in\Delta_{\calY_m}}\min_{y_m'\in\calY_m}L_{y_m'}q + v^\top q - v^\top \mu_m.
\end{equation*}
We recall that the calibration function satisfies $\zeta(\varepsilon) \geq \varepsilon/C$ iff $\delta \ell(y, \mu) \leq C\delta s(v,\mu)$, for all $v, \mu$ and $y$ among the minimizers of the surrogate.
Hence, for all $v, \mu$ and $y$ among the minimizers of the surrogate:
\begin{align*}
     (\max_{m\in[M]}C_m)\delta s(v,\mu) &= (\max_{m\in[M]}C_m)\sum_{m=1}^M\delta s_m(v_m, \mu_m)\big) \geq  \sum_{m=1}^M C_m\delta s_m(v_m, \mu_m) \\
     & \geq \sum_{m=1}^M\delta\ell_m(y_m, \mu_m) = \delta\ell(y,\mu).
\end{align*}
\end{proof}

\begin{proposition}[Calibration function for high-order factor graphs (\cref{ex:factorgraph})] Assume \eqref{eq:assumptiononL}. The constant $D$ from \cref{th:calibrationfunction-quantitative} for embeddings for unary and high-order interactions is the same as the constant $D$ with only unary potentials.
\end{proposition}
\begin{proof}
As the loss is decomposable as $L(y,y') = \frac{1}{M}\sum_{m=1}^ML_m(y_m, y_m')$, it only depends on the unary embeddings. This means that the constraint $A^\top\Expect_{z\sim\alpha}\phi(z) = A^\top\Expect_{z'\sim\beta}\phi(z')$ from \cref{eq:definitionDy} only affects the unary embeddings, and so the lower bound is the same.
\end{proof}

\section{Sharp Generalization Bounds for Regularized Objectives} \label{app:regularizedobjectives}\label{app:generalization}

For $\lambda > 0$ and $g \in \calG$ where $\calG$ is a vector valued reproducing kernel Hilbert space and with norm $\|\cdot\|_\calG$ and $g(x) \in \R^k$ defined as $g(x)_i  = \langle g, \Psi_i(x) \rangle_\calG$ with $\Psi_i: X \to \calG$ for $i = 1,\dots,k$. Note that in particular we have the identity $$K(x,x') = \langle \Psi(x), \Psi(x') \rangle_\calG \in \R^{k \times k},$$ where $K$ is the associated {\em vector-valued reproducing Kernel}. A simple example is the following. Let $K_0:X\times X \to \R$ be a scalar reproducing kernel, then the kernel $K(x,x') = \frac{1}{k}K_0(x,x') I_{k\times k}$ is a vector-valued reproducing kernel whose associated vector-valued reproducing kernel Hilbert space contains functions of the form $g: X \to \R^k$.
Now define
$$ \calR(g) = \mathbb{E}_{(x,y) \sim \rho} S(g(x),y), \quad \calR^\lambda(g) = \calR(g) + \lambda \|g\|^2_\calG, \quad g^\lambda = \argmin_{g \in \calG} \calR^\lambda(g).$$
Define also the empirical versions given a dataset $(x_i,y_i)_{i=1}^n$
$$ \calR_n(g) = \frac{1}{n} \sum_{i=1}^n S(g(x_i),y_i), \quad \calR^\lambda_n(g) = \calR_n(g) + \lambda \|g\|^2_\calG, \quad g^\lambda_n = \argmin_{g \in \calG} \calR^\lambda_n(g).$$
We will use the following theorem that is a slight variation of Thm.1 from \cite{sridharan2009fast}. 
\begin{theorem}\label{thm:gen-bound-approx}
Let $\delta \in (0,1)$. Let $L$ be the Lipschitz constant of $S$ and let $\|\Psi_i(x)\| \leq B$ for all $x \in X$, $i = 1,\dots, k$. Assume that there exists $g^\star$ such that $\calR(g^\star) = \inf_{g \in \calG} \calR(g)$. 
For any $g \in \calG$ the following holds
\begin{align}
    \calR(g) - \calR(g^\star) &\leq 2(\calR_n^\lambda(g) - \calR_n^\lambda(g_n^\lambda)) + \frac{16L^2B^2(32 + \log(1/\delta))}{\lambda n}  + \frac{\lambda}{2}\|g\|_{\calG}^2.
\end{align}
with probability $1-\delta$.
\end{theorem}
\begin{proof}
We apply the following error decomposition:
\begin{align*}
    \calR(g) - \calR(g^\star) &= 
    (\calR^\lambda(g) - \calR^\lambda(g^\lambda))
    + (\calR^\lambda(g^\lambda) - \calR^\lambda(g^\star))
    + \frac{\lambda}{2}(\|g^\star\|_{\calG}^2 - \|g\|_{\calG}^2).
\end{align*}
By applying Theorem~1 of \cite{sridharan2009fast} on $\calR^\lambda(g) - \calR^\lambda(g^\lambda)$, we have
\begin{equation*}
    \calR^\lambda(g) - \calR^\lambda(g^{\lambda}) \leq 2(\calR_n^\lambda(g) - \calR_n^\lambda(g_n^\lambda)) + \frac{16L^2B^2(32 + \log(1/\delta))}{\lambda n}.
\end{equation*}
Considering that $\calR^\lambda(g^\lambda) - \calR^\lambda(g^\star) \leq 
0$ by definition of $g^\lambda$, we obtain the desired result.
\end{proof}

Now we are ready to prove Theorem~3.4

\paragraph{Proof of Theorem 3.4.}
Apply the theorem above with $g = g^\lambda_n$. Note moreover that by Fisher consistency in Theorem~\ref{th:fisherconsistency},~$g^\star(x) = -A^\top \phi(f^\star(x))$. Let $\calG$ be the vector-valued reproducing kernel Hilbert space associated to the vector-valued kernel $K$, then
$B = \sup_{x\in X} \|\Psi(x)\|_\calG = \|K(x,\cdot)\|_\calG = \sup_{x\in X} \operatorname{Tr}(K(x,x))^{1/2}$, where $\operatorname{Tr}$ is the trace.
Moreover we have $L \leq 2\|A\| \max_{y \in \calY} \|\phi(y)\|$. The result is obtained by minimizing the resulting upper bound in $\lambda$ and then applying comparison inequality of Theorem~\ref{th:comparisoninequality}.
\qed{}

$ $

To conclude we extend a result from \cite{pillaud2017exponential} to our case. In the following assume $X = \R^d$ and denote by $\calG_m$ the vector-valued reproducing kernel induced by $K_m(x,x') = \frac{1}{k} \bar{K}_m(x,x') I_{k\times k}$ where $I_{k \times k}$ is the identity matrix and $\bar{K}_m(x,x')$ is the scalar kernel associated to the Sobolev space $W^m_s(\R^d)$ for $m > d/2$. Note that when $m = (d+1)/2$, $\bar{K}_m(x,x') = e^{-\|x-x'\|}$. 
\begin{theorem}\label{thm:conditions-fstar-in-G}
Let $X = \R^d$ and $\rho$ be such that $\overline{X_y} \cap \overline{X_{y'}} = \emptyset$, for every $y \neq y'$ where $X_y = \{x \in X ~|~ y \in \argmin_{z \in \calY} \mathbb{E}_{y\sim \rho(\cdot|x)} L(z,y)\}$. When $\calG \subseteq \calG_m$ for $m > d/2$, we have that
$$\|\phi(f^\star)\|_\calG < \infty.$$
\end{theorem}
\begin{proof}
Since $W^m_2(\R^d)$ contains the smooth and compactly supported functions $C^\infty_c(\R^d)$ by construction for any $m > 0$ and $\calG_m = W^m_2(\R^d)^{\otimes k}$, when $\calG \supseteq \calG_m$ we have that $\calG$ contains all the vector valued compactly supported smooth functions. Now note for any two sets $A, B$ there exists a compactly supported smooth function $f_{A,B}$ that has value $1$ on $A$ and~$0$ on~$B$ (see \cite{pillaud2017exponential} for more details). Now we build 
$$g = -A^\top \sum_{y \in \calY}\phi(y) f_{X_y, \cup_{z\neq y} X_z}.$$
Note that 
$d \circ g \in \argmin_{f:X \to \calY}~\mathbb{E}_{(x,y)\sim \rho} L(f(x),y)$,
since for any $x \in \cup_{y \in X_y} = \operatorname{support}(\rho)$ we have $d \circ g \in \argmin_{z \in \calY} \mathbb{E}_{y\sim \rho(\cdot|x)} L(z,y)$ by construction. I.e. $d \circ g = f^\star$ and so $g = -A^\top\phi(f^\star)$. To conclude the theorem, note that~$\|g\|_{\calG} < \infty$ since
$g \in C^\infty_c(\R^d,\R^k) \subseteq \calG_m \subseteq \calG$.
\end{proof}
\section{Max-min margin and dual formulation}\label{app:dualform}

\subsection{Derivation of the Dual Formulation}
Let us first define
\begin{equation*}
    H_i(\mu, w) = \min_{y\in\calY}\phi(y)^\top A\mu + g_w(x_i)^\top \mu - g_w(x_i)^\top\phi(y_i).
\end{equation*}
Let's denote $ w(\mu) = \frac{1}{\lambda n}\Phi_n(\mu - \phi_n)$ where $\Phi_n = \frac{1}{\lambda n}(\Phi(x_1),\ldots,\Phi(x_n))$ is the $d\times n$ scaled input data matrix and~$\phi_n = (\phi(y_1),\ldots,\phi(y_n))^\top $ is the $n\times k$ output data matrix. Note that $w^\top w(\mu) = \frac{1}{\lambda n}\sum_{i=1}^ng_w(x_i)^\top (\mu_i - \phi(y_i))$.
The dual formulation $\textbf{(D)}$ of $\operatorname{M^4N}$s can be derived as follows:
\begin{align*}
    &\min_{w \in {\cal G}}~\frac{1}{n}\sum_{i=1}^n\max_{\mu_i\in\calM}~H_i(\mu_i,  w) + \frac{\lambda}{2}\|w\|_{\cal G}^2 \\
    &=\min_{w \in {\cal G}}~\frac{1}{n}\Big(\sum_{i=1}^n\max_{\mu_i\in\calM}~g_w(x_i)^\top (\mu_i - \phi(y_i)) + \min_{  y'\in\calY}\phi(y')^\top A\mu_i\Big) + \frac{\lambda}{2}\|w\|_{\cal G}^2 \\
    &=\max_{  \mu\in\calM\times\cdots\times\calM}~\frac{1}{n}\sum_{i=1}^n\Big(\min_{w \in {\cal G}}~g_w(x_i)^\top (\mu_i - \phi(y_i)) + \min_{  y'\in\calY}\phi(y')^\top A\mu_i\Big) + \frac{\lambda}{2}\|w\|_{\cal G}^2 \\
    &=\max_{  \mu\in\calM\times\cdots\times\calM}~ \frac{1}{n}\sum_{i=1}^n\min_{  y'\in\calY}\phi(y')^\top A\mu_i + \min_{w \in {\cal G}} ~-w^\top \frac{1}{n}w(\mu) + \frac{\lambda}{2}\|w\|_{\cal G}^2 \\
    &=\max_{  \mu\in\calM\times\cdots\times\calM}~ ~\frac{1}{n}\sum_{i=1}^n\min_{  y'\in\calY}\phi(y')^\top A\mu_i + \lambda\min_{w \in {\cal G}}~ -w^\top w(\mu) + \frac{1}{2}\|w\|_{\cal G}^2 \\
    &=\max_{  \mu\in\calM\times\cdots\times\calM}~ \frac{1}{n}\sum_{i=1}^n \min_{y'}\phi(y')^\top A\mu_i -\frac{\lambda}{2}\|w(\mu)\|_2^2,
\end{align*}
where the maximization and minimization have been interchanged using strong duality. We have~$w^\star = w(\mu^\star)$. 

\subsection{Computation of the Dual Gap}

The dual gap $g$ at the pair $(w(\mu), \mu)$ decomposes additively in individual dual gaps as
$g(w, \mu) = \frac{1}{2}\sum_{i=1}^ng_i(w, \mu_i)$:
\begin{align*}\label{eq:dualgaps}
   g(w, \mu_i) &= \frac{1}{n}\sum_{i=1}^n\max_{\mu_i'\in\calM}~H_i(\mu_i',  w) + \frac{\lambda}{2}w^\top w - \Big(\frac{1}{n}\sum_{i=1}^n \min_{y'}\phi(y')^\top A\mu_i -\frac{\lambda}{2}w^\top w\Big) \\
   &= \frac{1}{n}\sum_{i=1}^n \Big(\max_{\mu_i'}H_i(\mu_i',  w) - \min_{y'}\phi(y')^\top A\mu_i\Big) + \lambda w^\top w \\
   &=\frac{1}{n}\sum_{i=1}^n \Big(\max_{\mu_i'}H_i(\mu_i',  w) - \min_{y'}\phi(y')^\top A\mu_i\Big) + \frac{1}{n}\sum_{i=1}^n w^\top (\lambda nw_i(\mu_i)) \\
   &=\frac{1}{n}\sum_{i=1}^n \max_{\mu_i'}H_i(\mu_i',  w) + w^\top (\lambda nw_i(\mu_i)) - \min_{y'}\phi(y')^\top A\mu_i \\
   &= \frac{1}{n}\sum_{i=1}^n \max_{\mu_i'}H_i(\mu_i',  w) -  H_i(\mu_i, w) = \frac{1}{n}\sum_{i=1}^n g_i(w, \mu_i),
\end{align*}
where $w_i(\mu_i) = \frac{1}{\lambda n}\Phi(x_i)(\mu_i - \phi(y_i))^\top$.

\section{Generalized Block-Coordinate Frank-Wolfe}\label{app:gbcfw}
\subsection{General Convergence Result}\label{app:generalconvergenceresult}

In order to prove  a convergence bound, following~\cite{lacoste2012block}, we will consider a more general optimization problem, and combine their proof with the proof of generalized conditional gradient from~\cite{bach2015duality}, with an additional support from approximate oracles. 

We thus consider a product domain $\mathcal{K} = \mathcal{K}_1 \times \cdots \mathcal{K}_n$, and a smooth function $f$ defined on $\mathcal{K}$, as well as $n$ functions~$h_1,\dots,h_n$. We assume that $f$ is $L_i$-smooth with respect to the $i$-th block. The optimization problem reads
\begin{equation}\label{eq:optimproblemapp}
    \underset{z\in\calK_1\times\cdots\times\calK_n}{\min}~g(z)\defeq f(z) + \sum_{i=1}^nh_i(z_i).
\end{equation}

The algorithm, described in \cref{alg:gbcfwgen}, proceeds as follows. Starting from $z^{(0)}\in \calK_1\times\cdots\times\calK_n$, for $t\geq 0$, select~$i(t)$ uniformly at random and find $\bar{z}_{i(t)} \in \mathcal{K}_{i(t)}$ such that minimizes a convex lower bound of the objective function on the~$\calK_{i(t)}$'th block. This convex lower bound is constructed by linearizing only the smooth part of the objective function. Hence, the minimization of the lower bound reads:
\begin{equation}\label{eq:apporacleapp}
    h_{i(t)}(\bar{z}_{i(t)}) + \nabla_{i(t)} f(z^{(t)})^\top \bar{z}_{i(t)} \leqslant \inf_{z_{i(t)} \in \mathcal{K}_{i(t)}} h_{i(t)}(z_{i(t)}) + \nabla_{i(t)} f(z^{(t)})^\top \ {z}_{i(t)} + \varepsilon_t.
\end{equation}
Note that we allow an error of at most $\varepsilon_t$ on the computation of the generalized Frank-Wolfe oracle. This is key in our analysis as in our setting we only have access to an approximate oracle. Finally, define $z^{(t+1)}$ by copying $z^{(t)}$ except the~$i(t)$-th coordinate, which is taken to be 
$$z^{(t+1)}_{i(t)} = (1 - \gamma_t)z^{(t)}_{i(t)}
+ \gamma_t\bar{z}_{i(t)} .$$

We have, using the convexity of $h_{i(t)}$ and the smoothness of $f$, and denoting $z^\ast$ a minimizer of $h(z) = f(z) + \sum_{i=1}^n h_i(z_i) $:
\begin{eqnarray*}
& & f(z^{(t+1)}) + \sum_{i=1}^n h_i(z^{(t+1)}_i) \\
& \leq & f(z^{(t)}) + (z^{(t+1)}_{i(t)}- z^{(t)}_{i(t)})^\top \nabla_{i(t)} f(z^{(t)})+ \frac{L_{i(t)}}{2} \| z^{(t+1)}_{i(t)}- z^{(t)}_{i(t)}\|^2 \\
& & \hspace*{2cm} + \sum_{i=1}^n h_i(z^{(t)}_i)
+ h_{i(t)}(z^{(t+1)}_{i(t)}) - h_{i(t)}(z^{(t)}_{i(t)})\\
& \leq & f(z^{(t)})  + \gamma_t (\bar{z}_{i(t)}- z^{(t)}_{i(t)})^\top \nabla_{i(t)} f(z^{(t)})+ \frac{L_{i(t)}}{2} \gamma_t^2 \| \bar{z}_{i(t)}- z^{(t)}_{i(t)}\|^2\\
& & \hspace*{2cm} + \sum_{i=1}^n h_i(z^{(t)}_i)
+ (1-\gamma_t) h_{i(t)}(z^{(t)}_{i(t)})  + \gamma_t
h_{i(t)}(\bar{z}_{i(t)})- h_{i(t)}(z^{(t)}_{i(t)}) \\
& \leq & f(z^{(t)}) + \sum_{i=1}^n h_i(z^{(t)}_i)
+ \gamma_t \big[h_{i(t)}(\bar{z}_{i(t)}) 
+ \bar{z}_{i(t)}^\top \nabla_{i(t)} f(z^{(t)})\big]  \\
& & \qquad \qquad -\gamma_t\big[h_{i(t)}(z^{(t)}_{i(t)}) + (z^{(t)}_{i(t)})^\top \nabla_{i(t)} f(z^{(t)}) \big]+ \frac{L_{i(t)}}{2} \gamma_t^2 {\rm diam}(\mathcal{K}_{i(t)})^2\\
\end{eqnarray*}
Now, we use \cref{eq:apporacleapp}, i.e., the fact that $\bar{z}_{i(t)}$ is an approximate solution of $\inf_{z_{i(t)} \in \mathcal{K}_{i(t)}} \big\{ h_{i(t)}(z_{i(t)}) + \nabla_{i(t)} f(z^{(t)})^\top \ {z}_{i(t)} \big\}$. In this case, we can continue upper bounding the above quantity as

\begin{algorithm}[tb]
   \caption{Generalized Block-Coordinate Frank-Wolfe}
   \label{alg:gbcfwgen}
\begin{algorithmic}
   \STATE Let $z^{(0)}\in \calK_1\times\cdots\times\calK_n$
   \FOR{$t=0$ {\bfseries to} $T$}
   \STATE Pick $i(t)$ at random in $\{1,\ldots, n\}$
   \STATE $\bar{z}_{i(t)}^\star\in\argmin_{z_{i(t)}'\in\calK_i} \nabla_{i(t)}f(z^{(t)})^\top z_{i(t)}' + h_{i(t)}(z_{i(t)}')$ \hspace{1cm} \textit{(solve generalized Frank-Wolfe oracle)}.
   \STATE $\gamma \defeq \frac{2n}{t + 2n}$ or optimize $\gamma$ by line-search.
   \STATE $z_{i(t)}^{(t+1)}\defeq (1 - \gamma)z_{i(t)}^{(t)} + \gamma \bar{z}_{i(t)}^\star$ \hspace{1cm} \textit{(Only affecting the i-th coordinate. Copy the rest.)} \\
   \ENDFOR
\end{algorithmic}
\end{algorithm}

\begin{eqnarray*}
& \leq & f(z^{(t)}) + \sum_{i=1}^n h_i(z^{(t)}_i)
+ \gamma_t \big[\inf_{z_{i(t)} \in \mathcal{K}_{i(t)}} \big\{ h_{i(t)}(z_{i(t)}) + \nabla_{i(t)} f(z^{(t)})^\top \ {z}_{i(t)} \big\}\big] \\
& & \hspace*{2cm}
-\gamma_t \big[
h_{i(t)}(z^{(t)}_{i(t)}) - (z^{(t)}_{i(t)})^\top \nabla_{i(t)} f(z^{(t)}) \big] + \frac{L_{i(t)}}{2} \gamma_t^2 {\rm diam}(\mathcal{K}_{i(t)})^2 + \gamma_t   \varepsilon_t\\
& \leq & f(z^{(t)}) + \sum_{i=1}^n h_i(z^{(t)}_i)
+ \gamma_t \big[  h_{i(t)}(z_{i(t)}^\ast) + \nabla_{i(t)} f(z^{(t)})^\top \ {z}_{i(t)}^\ast  \\
& & \hspace*{2cm}
-
h_{i(t)}(z^{(t)}_{i(t)}) - (z^{(t)}_{i(t)})^\top \nabla_{i(t)} f(z^{(t)}) \big] + \frac{L_{i(t)}}{2} \gamma_t^2 {\rm diam}(\mathcal{K}_{i(t)})^2 + \gamma_t   \varepsilon_t.
\end{eqnarray*}
Let's now define $C\geq 0$ as
\begin{equation*}
    C = \sum_{i=1}^nL_i{\rm diam}(\mathcal{K}_{i})^2.
\end{equation*}
Finally, if we denote by $\mathcal{F}_t$ the information up to time $t$, we have that
\begin{eqnarray*}
& & \mathbb{E} \Big[ f(z^{(t+1)}) + \sum_{i=1}^n h_i(z^{(t+1)}_i)~|~\mathcal{F}_t \Big]  \\
& \leq & 
f(z^{(t)}) + \sum_{i=1}^n h_i(z^{(t)}_i) \\
&+ & \gamma_t \big[  \frac{1}{n} \sum_{i=1}^n h_{i}(z_{i}^\ast) + \frac{1}{n} \nabla f(z^{(t)})^\top \ {z}^\ast  
- \frac{1}{n} \sum_{i=1}^n
h_{i}(z^{(t)}_{i}) -  \frac{1}{n} (z^{(t)})^\top \nabla  f(z^{(t)}) \big]\\
&+ & \frac{\gamma_t^2}{2}\frac{C}{n}  + \gamma_t   \varepsilon_t,
\end{eqnarray*}
 where we have used $C/n=\Expect[L_{i(t)}{\rm diam}(\mathcal{K}_{i(t)})^2]$, $\Expect \big[h_{i(t)}(z_{i(t)})\big] = \frac{1}{n} \sum_{i=1}^n h_{i}(z_{i})$ and $\Expect[z_{i(t)}] = \frac{1}{n}z$ for any $z\in\calK_1\times\cdots\times\calK_n$. Thus, if $g$ is the objective function as defined in \eqref{eq:optimproblemapp}, we get:
\begin{eqnarray*}
 \mathbb{E}\big[g (z^{(t+1)})  - g(z^\ast)\big]
 & \leq & ( 1 - \frac{\gamma_t}{n}) \big[  \mathbb{E}g (z^{(t)})  - g(z^\ast) \big]
  + \frac{\gamma_t^2}{2}\frac{C}{n} + \gamma_t   \varepsilon_t.
 \end{eqnarray*}
 Note that the above inequality is the same appearing in \cite{jaggi2013revisiting} but with the key difference of the factor $1/n$, which stems from the random block-coordinate procedure of the algorithm. If we define $G_t = \mathbb{E}\big[g (z^{(t)})  - g(z^\ast)\big]$, we can re-write the recursion as 
 \begin{equation*}
     G_{t+1} \leq ( 1 - \frac{\gamma_t}{n})G_t + \frac{\gamma_t^2C}{2n}  +\gamma_t\varepsilon_t.
 \end{equation*}
 Let's first set $\varepsilon_t = 0$, i.e., $G_{t+1} \leq ( 1 - \frac{\gamma_t}{n})G_t + \frac{\gamma_t^2C}{2}$, and prove by induction if $\gamma_t = \frac{2n}{t + 2n}\in[0,1]$, we obtain
 \begin{equation*}
     G_t \leq \frac{2n(C + G_0)}{t + 2n} \hspace{1cm} t\geq 0.
 \end{equation*}
 Let's proceed by induction. The \textit{base-case} $k=0$ is satisfied as $C\geq 0$. 
 \begin{align*}
     G_{t+1} &\leq ( 1 - \frac{\gamma_t}{n})G_t + \frac{\gamma_t^2C}{2n} \\
    &  =  ( 1 - \frac{2}{t + 2n})G_t + (\frac{2n}{t + 2n})^2\frac{C}{2n} \\
    &\leq ( 1 - \frac{2}{t + 2n})\frac{2n(C + G_0)}{t + 2n} + (\frac{1}{t + 2n})^2 2nC
 \end{align*}
 Rearranging the terms gives
 \begin{align*}
     G_{t+1} &\leq \frac{2nC}{t + 2n}(1 - \frac{2}{t + 2n} + \frac{1}{t + 2n}) \\
     &= \frac{2nC}{t + 2n}\frac{t + 2n - 1}{t + 2n} \\
     &\leq \frac{2nC}{t + 2n}\frac{t + 2n}{t + 2n + 1} \\
     &= \frac{2nC}{t + 2n + 1},
 \end{align*}
 which is the claimed bound for $k+1$.   
 If we now we use an error 
 \begin{equation}\label{eq:formoferror}
     \varepsilon_t = \frac{1}{2}\delta\gamma_tL_{i(t)}{\rm diam}(\mathcal{K}_{i(t)})^2.
 \end{equation}
 Then, we have that 
 \begin{equation*}
     G_{t+1} \leq ( 1 - \frac{\gamma_t}{n})G_t + \frac{\gamma_t^2C(1 + \delta)}{2n}C,
 \end{equation*}
 and so we get
 \begin{equation*}
     \mathbb{E}\big[g (z^{(t+1)})  - g(z^\ast)\big] \leq \frac{2n}{t + 2n}\left(\mathbb{E}\big[g (z^{(0)})  - g(z^\ast)\big] + (1 + \delta)\sum_{i=1}^nL_i{\rm diam}(\mathcal{K}_{i})^2\right).
 \end{equation*}
 
 In order to obtain the final bound only in terms of $\sum_{i=1}^nL_i{\rm diam}(\mathcal{K}_{i})^2$, we can reuse the techniques from~\cite{lacoste2012block}, such as a single batch generalized Frank-Wolfe step, or use line search instead of constant step-sizes. Using these techniques, we  can manage to set
 \begin{equation*}
     \mathbb{E}\big[g (z^{(0)})  - g(z^\ast)\big] \leq n\max_i{\rm diam}(\calK_i)^2\frac{\max_iL_i}{2},
 \end{equation*}
 so that
we obtain 
 \begin{equation*}
     \mathbb{E}\big[g (z^{(t+1)})  - g(z^\ast)\big] \leq (2 + \delta)\frac{2n^2}{t + 2n}\max_iL_i\max_i{\rm diam}(\calK_i)^2.
 \end{equation*}
\subsection{Application to Our Setting, Proof of Theorem 5.1. } 

\begin{algorithm}[htb!]
   \caption{Generalized Block-Coordinate Frank-Wolfe}
   \label{alg:gbcfwapp}
\begin{algorithmic}
   \STATE Let $ w^{(0)}\defeq  w_i^{(0)}\defeq \bar{ w}^{(0)} \defeq 0$
   \FOR{$t=0$ {\bfseries to} $T$}
   \STATE Pick $i$ at random in $\{1,\ldots, n\}$
   \STATE $(\mu_i^\star,\nu_i^\star) = \calO^{\varepsilon_t}(g_{w^{(t)}}(x_i))$ (solve oracle with precision $\varepsilon_t$)
   \STATE $ w_s \defeq \Phi_n(\mu_i^\star - \phi_n) / (\lambda n)$ \\
   \STATE $\gamma \defeq \frac{2n}{t + 2n}$ (or line-search)
   \STATE $ w_i^{(t+1)} \defeq (1-\gamma) w_i^{(t)} + \gamma w_s$
   \STATE $ w^{(t+1)}\defeq  w^{(t)} +  w_i^{(t+1)} -  w_i^{(t)}$
   \STATE (Optional averaging: $\bar{ w}^{(t+1)}\defeq \frac{t}{t+2}\bar{ w}^{(t)} + \frac{2}{t+2} w^{(t+1)}$).
   \ENDFOR
\end{algorithmic}
\end{algorithm}

In our setting, we have that ${\rm diam}(\calK_i) = \operatorname{diam}(\calM)$ and $L_i\leq \frac{R^2}{\lambda n^2}$ ($R$ is the maximal norm of features). Hence, the bound simplifies to
\begin{equation*}
    \mathbb{E}\big[g (z^{(t+1)})  - g(z^\ast)\big] \leq \frac{2(2 + \delta)}{t + 2n}\frac{R^2{\rm diam}(\calM)^2}{\lambda}.
\end{equation*}
 Which means that in order to get $\mathbb{E}\big[g (z^{(t+1)})  - g(z^\ast)\big]\leq\varepsilon$ one needs 
 \begin{equation*}
     t\geq \frac{2(2+\delta)R^2\operatorname{diam}(\calM)^2}{\lambda\varepsilon} + 2n = O\left(n + \frac{R^2\operatorname{diam}(\calM)^2}{\lambda\varepsilon}\right)
 \end{equation*}
 iterations.

\section{Solving the Oracle with Saddle Point Mirror Prox} \label{app:oracle}
Let $\calX\subset\Rspace{k}$, $\calY\subset\Rspace{k}$ be compact and convex sets. Let $F:\calX\times\calY\rightarrow\Rspace{}$ be a continuous function such that $F(\cdot, y)$ is convex and $F(x, \cdot)$ is concave. We are interested in computing 
\begin{equation*}
    \min_{x\in\calX}\max_{y\in\calY}~F(x,y).
\end{equation*}
By Sion's minimax theorem there exists a pair $(x^\star, y^\star)\in\calX\times\calY$ such that 
\begin{equation*}
    F(x^\star, y^\star) = \min_{x\in\calX}\max_{y\in\calY}~F(x,y) = \max_{y\in\calY}\min_{x\in\calX}~F(x,y).
\end{equation*}

We assume that
\begin{align*}
    &\|\nabla_x F(x,y) - \nabla_x F(x', y)\|_{\calX}^* \leq  \beta_{1,1}\|x-x'\|_{\calX} \\
    &\|\nabla_x F(x,y) - \nabla_x F(x, y')\|_{\calX}^* \leq  \beta_{1,2}\|y-y'\|_{\calY} \\ 
    &\|\nabla_y F(x,y) - \nabla_y F(x', y)\|_{\calY}^* \leq \beta_{2,1}\|x-x'\|_{\calX} \\
    &\|\nabla_y F(x,y) - \nabla_y F(x, y')\|_{\calY}^* \leq  \beta_{2,2}\|y-y'\|_{\calY}, \\ 
\end{align*}
where $\|\cdot\|_{\calX}^*, \|\cdot\|_{\calY}^*$ denote the dual norms of $\|\cdot\|_{\calX},\|\cdot\|_{\calY}$, respectively. We are interested in finding an algorithm that produces $(\widehat{x}, \widehat{y})$ that has small \textit{duality gap} $g(\widehat{x}, \widehat{y})$ defined as
\begin{equation*}
   g(\widehat{x}, \widehat{y}) \defeq \max_{y\in\calY}~F(\widehat{x}, y) - \min_{x\in\calX}~F(x, \widehat{y}).
\end{equation*}

\subsection{Saddle Point Mirror Prox (SP-MP)} \label{app:spmp}

Define $H_{\calX}:\calD_{\calX}\rightarrow\Rspace{}$ and $H_{\calY}:\calD_{\calY}\rightarrow\Rspace{}$, which are $1$-strongly concave w.r.t a norm $\|\cdot\|_{\calX}$ on $\calX\cap\calD_{\calX}$ and 
$\|\cdot\|_{\calY}$ on~$\calY\cap\calD_{\calY}$, respectively. Denote $R_{\calX} = \sup_{x\in\calX}H_{\calX}(x) - \min_{x\in\calX}H_{\calX}(x)$ and $R_{\calY}$ similarily for $H_{\calY}$. Define $\calZ \defeq \calX\times\calY$ and~$H:\calD\defeq\calD_{\calX}\times\calD_{\calY}\rightarrow\Rspace{}$ defined as $H(z) = \frac{1}{R_{\calX}^2}H_{\calX}(x) + \frac{1}{R_{\calY}^2}H_{\calY}(y)$, where $z=(x,y)$. The saddle point mirror prox (SP-MP) algorithm is defined as follows.

Start with $z^{(1)} = (x^{(1)}, y^{(1)}) = \argmax_{z\in\calZ}~H(z)$.
Then at every iteration $k$:

\begin{align*}
    (u^{(k+1)}, v^{(k+1)}) &\defeq \argmin_{z\in\calZ\cap\calD}~\eta(\nabla_x F(x^{(k)}, y^{(k)}), -\nabla_y F(x^{(k)}, y^{(k)}))^\top z + D_{-H}(z, z^{(k+1)}) \\
    (x^{(k+1)}, y^{(k+1)}) &\defeq \argmin_{z\in\calZ\cap\calD}~\eta(\nabla_x F(u^{(k+1)}, v^{(k+1)}), -\nabla_y F(u^{(k+1)}, v^{(k+1)}))^\top z + D_{-H}(z, z^{(k+1)})
\end{align*}

The following \cref{th:appspmp} by \cite{nemirovski2004prox} studies the convergence of SP-MP.

\begin{theorem}[\cite{nemirovski2004prox}] \label{th:appspmp}
Let $L = \max(\beta_{11}R_{\calX}^2, \beta_{22}R_{\calY}^2, \beta_{12}R_{\calX}R_{\calY}, \beta_{21}R_{\calX}R_{\calY})$. Then,
the algorithm saddle point mirror prox (presented at the beginning of the section) runned with 
$\eta = \frac{1}{2L}$ satisfies
\begin{equation*}
    g(\bar{u}^{K}, \bar{v}^{K}) \leq \frac{4L}{K},
\end{equation*}
where $\bar{u}^{K}\defeq \frac{1}{K}\sum_{k=1}^K u^{(k)}$ and $\bar{v}^{K}\defeq\frac{1}{K}\sum_{k=1}^K v^{(k)}$.
\end{theorem}
In our setting, we have that $\calX=\calY=\calM$ and 
\begin{equation}\label{eq:bilinearfunc}
    F(\nu, \mu) = \nu^\top A\mu + v^\top \mu.
\end{equation}
The gradients have the following form:
\begin{equation*}
    \nabla_\nu F(\nu, \mu) = A\mu, \hspace{0.3cm}\text{and}\hspace{0.3cm}  \nabla_{\mu} F(\nu, \mu) = A^\top \nu + v.
\end{equation*}

\subsection{Max-Min Oracle for Sequences (special case of \cref{ex:factorgraph})}
Consider unary potentials and binary potentials between adjacent variables. The embeddings can be written as 
\begin{equation*}
    \phi(y) = (\phi_{u}(y), \phi_{p}(y)) =  ((\phi_m(y_m))_{m=1}^M, \phi_{m, m+1}(y_{m, m+1})_{m=1}^{M-1}) \in\Rspace{RM + R^2(M-1)},
\end{equation*}
where $\phi_m(y_m) = e_{y_m}\in\Rspace{R}$ and $\phi_{m,m+1}(y_{m,m+1}) = e_{y_{m,m+1}}\in\Rspace{R^2}$ are vectors of the canonical basis. Here, $\phi_u$ and $\phi_p$ stand for unary and pair-wise embeddings.
If the loss decomposes coordinate-wise as
$L(y, y') = \frac{1}{M}\sum_{i=1}^ML_m(y_m, y_m')$ as detailed in \cref{ex:factorgraph}, the loss decomposition reads
\begin{equation*}
A = \left(\begin{array}{@{}c|c@{}}
   \begin{matrix}
    L_1/M & \cdots & 0_{R\times R} \\
    \vdots & \ddots & \vdots \\
    0_{R\times R} & \cdots & L_M/M
    \end{matrix}
  & \bigzero_{MR\times (M-1)R^2} \\
\hline \bigzero_{(M-1)R^2\times MR} & \bigzero_{(M-1)R^2\times (M-1)R^2}
\end{array}\right), \hspace{1cm} a = 0.
\end{equation*}
The bilinear function \eqref{eq:bilinearfunc} takes the following form:
\begin{equation*}
    F(\nu, \mu) = \sum_{m=1}^M\nu_m^\top L_m \mu_m + \sum_{m=1}^Mv_m^\top\mu_m + \sum_{p=1}^{M-1}v_p^\top\mu_p.
\end{equation*}
Note that as $A$ is low-rank, the dependence on $\nu$ is only on the unary embeddings, which means that the minimization over~$\nu$ is over a simpler domain that decomposes as $\calQ=\Pi_{m=1}^M\Delta_R$.

We consider the entropies $H_\calQ:\calQ\rightarrow\Rspace{}$ and $H_{\calM}:\calM\rightarrow\Rspace{}$ defined as:
\begin{equation*}\label{eq:entropies}
    H_{\calQ}(\nu) \defeq \sum_{m=1}^MH_S(\nu_m), \hspace{1cm} 
    H_{\calM}(\mu) \defeq \left\{\begin{array}{lll}
        \max_{q\in\Delta_{|\calY|}} & H_S(q) & \\
        \text{s.t.} & \Expect_{y\sim q}\phi_m(y_m) = \mu_m, &1\leq m\leq M \\
         & \Expect_{y\sim q}\phi_p(y_p) = \mu_p, & 1\leq p \leq M-1 \\
    \end{array}\right. ,
\end{equation*}
where for $q\in\Delta_k$, we define the \textit{Shannon entropy} $H_S:\Delta_k\rightarrow\Rspace{}$ as $H_S(q) = -\sum_{j=1}^kq_j\log q_j$.

In order to apply SP-MP we need to compute two projections in $\calQ$ and $\calM$ with respect to the corresponding entropies described above. The update on $\nu\in\calQ$ takes the form
\begin{equation}\label{eq:spmpfactorgraphQ}
    \argmin_{\nu\in\calQ}~\eta \sum_{m=1}^M\nu_m^\top L_m \mu_m^{(t)} + D_{-H_{\calQ}}(\nu, \nu^{(t)}).
\end{equation}
As the entropy $H_{\calQ}$ is separable, the projection \eqref{eq:spmpfactorgraphQ} is separable and can be computed with the softmax operator. The update on $\mu\in\calM$ takes the form
\begin{equation}\label{eq:spmpfactorgraphM}
    \argmin_{\mu\in\calM}~-\eta \sum_{m=1}^M\mu_m^\top(L_m^\top \nu_m^{(t)} + v_m) -\eta\sum_{p=1}^{M-1}\mu_p^\top v_p + D_{-H_{\calM}}(\mu, \mu^{(t)}).
\end{equation}
Projection \eqref{eq:spmpfactorgraphM} can be computed using marginal inference using the sum-product algorithm.

\textbf{Norm $\boldsymbol{\|\cdot\|_{\calQ}}$ and constants $\boldsymbol{R_{\calQ},\sigma_\calQ}$.} We choose the norm as the $L_1$-norm 
$\|\nu\|_{\calQ}\defeq \|\nu\|_1 = \sum_{m=1}^M\|\nu_m\|_1$. From Pinsker's inequality, we know that $H(\nu_m)$ is 1-strongly convex with respect to $\|\cdot\|_1$ in $\Delta_{R}$. Hence, we have that $H_{\calQ}(\nu)$ is 1-strongly respect with respect to $\|\cdot\|_1$ in $\calQ$. Moreover, using that $\min_{q\in\calQ}H_{\calQ}(\nu) = 0$, we have that
\begin{equation*}
    R_{\calQ}^2 \defeq \max_{\nu\in\calQ}H_{\calQ}(\nu) = \max_{\nu\in\Pi_{m=1}^M\Delta_{R}}\sum_{m=1}^MH(\nu_m) = \sum_{m=1}^M\max_{\nu_m\in\Delta_{R}} H(\nu_m) = \sum_{m=1}^M\log R=M\log R.
\end{equation*}

\textbf{Norm $\boldsymbol{\|\cdot\|_{\calM}}$ and constants $\boldsymbol{R_{\calM},\sigma_\calM}$.} If we choose the $L_2$-norm $\|\mu\|_{\calM} \defeq \|\mu\|_2$, the strong-convexity constant of~$H_{\calM}:\calM\rightarrow\Rspace{}$ defined in \cref{eq:entropies} with respect to $\|\cdot\|_2$ is 
\begin{equation*}
    \sigma_\calM = \operatorname{diam}(\calM)^{-2}.
\end{equation*}
In order to see this, note that the strong-convexity parameter $\sigma_\calM$ of $H_{\calM}$ is equal to the inverse of the smoothness parameter of the partition function $A(v) = \log\big(\sum_{y\in\calY}\exp(\langle\phi(y), v\rangle)\big)$, which corresponds to the maximal dual norm $\|\cdot\|_{*}$ of the covariance operator $\Sigma(v) = \Expect_{y\sim q_v}\phi(y)\phi(y)^\top  - \Expect_{y\sim q_v}\phi(y)\Expect_{y\sim q_v}\phi(y)^\top $, where $q_v(y) = \exp\langle v, \phi(y)\rangle/\sum_{y'\in\calY}\exp(\langle\phi(y'), v\rangle)\big)$
If we consider $\|\cdot\|_2$, it follows directly that $ \sigma_{\calM}^{-1} = \sup_{v}\|\Sigma(v)\|_2 \leq \operatorname{diam}(\calM)^2$.
Finally, using that $\min_{\mu\in\calM}H_{\calM}(\mu) = 0$, we have that
\begin{equation*}
    R_{\calM}^2 \defeq \max_{\mu\in\calM}H_{\calM}(\mu) = \sum_{m=1}^M\max_{\mu_m\in\Delta_{R}}H(\mu_m) + (\leq 0) \leq \sum_{m=1}^M\log R.
\end{equation*}

\textbf{Computation of the smoothness constants $\boldsymbol{(\beta_{11}, \beta_{12},\beta_{21},\beta_{22})}$.}
\begin{itemize}
    \item[-] $\beta_{11}=0$ as $\nabla_m F_x(q, \mu)$ is constant in $q$ for all $m\in[M]$.
    \item[-] We have that $\|L_m(\mu_m - \mu_m')\|_{\infty} \leq \|L_m\|_{\infty}\|\mu_m - \mu_m'\|_1$. Hence, $\beta_{12} = \max_{m\in[M]}\|L_m\|_{\infty}$.
    \item[-] We have that $\nabla_m F_y(q, \mu)$ and $\nabla_c F_y(q, \mu)$ are constant in $\mu$ for all $m\in[M]$ and $c\in C$, so $\beta_{12} = 0$.
    \item[-] We have that $\|L_m^\top (q_m - q_m')\|_2 \leq \|L_m^\top \|_2\|q_m-q_m\|_2$. Hence, $\beta_{22} = \max_{m\in[M]}\|L_m^\top \|_2$.
\end{itemize}

Finally, the constant $L$ appearing in \cref{th:appspmp} reads
\begin{equation*}
    L = \max_{m\in[M]}\|L_m\|_2\operatorname{diam}(\calM)^2M\log R.
\end{equation*}
\subsection{Max-Min Oracle for Ranking and Matching of \cref{ex:matching}}
We represent the permutation $\sigma\in \calS_M$ using the corresponding permutation matrix $\phi(\sigma) = P_{\sigma}\in\Rspace{M\times M}$. The loss decomposition is
\begin{equation*}
    L(\sigma, \sigma') = \frac{1}{M}\sum_{m=1}^M 1(\sigma(j)\neq \sigma'(j)) = 1 - \frac{\langle P_{\sigma}, P_{\sigma'}\rangle}{M} = 1 - \frac{\langle \phi(\sigma), \phi(\sigma')\rangle}{M},
\end{equation*}
i.e., $A = -Id/M$ and $a=1$. The marginal polytope $\calM$ corresponds to the \textit{Birkhoff polytope} or equivalently, the polytope of \textit{doubly stochastic matrices}
\begin{equation*}
    \calM = \operatorname{hull}\{P_{\sigma}~|~\sigma\in\calS_M\} = \{P\in\Rspace{M\times M}~|~P1=1,P^T1=1, 0\leq P_{ij}\leq 1, i,j\in[M]\}.
\end{equation*}

The max-min oracle corresponds to the following saddle-point problem:
\begin{equation}\label{eq:maxminmatching}
   \argmax_{P\in\calM}\min_{Q\in\calM}~\langle S, P\rangle - \langle Q, P\rangle / m.
\end{equation}

We have three natural options for the entropy, namely, the constrained Shannon entropy (which is the one used in the factor graph \cref{ex:factorgraph}), the entropy of marginals and the quadratic entropy.

\textbf{Constrained Shannon Entropy.} In this case, 
\begin{equation*}
    H(Q) \defeq \max_{p\in\Delta_{\calS_M}} -\sum_{\sigma\in\calS_M}p(\sigma)\log p(\sigma) \quad\text{s.t.}\quad \sum_{\sigma\in\calS_M}p(\sigma) P_{\sigma} = Q.
\end{equation*}
The projection corresponds to marginal inference, which is in general $\# P$-complete as we have to compute the permanent~\cite{valiant1979complexity}. As noted by \cite{petterson2009exponential}, it can be `efficiently' computed exactly up to $M=30$ with complexity $O(M2^M)$ using an algorithm by \cite{ryser1963combinatorial}. Note that this is way faster than enumeration which is of the order of~$M!\sim M^M$.

\textbf{Entropy of Marginals.} We can define the entropy defined in the marginals as
\begin{equation}\label{eq:entropyonmarginals}
    H(Q) = -\sum_{i,j=1}^MQ_{ij}\log Q_{ij}.
\end{equation}
The projection can be computed up to precision $\delta$ using the \textit{Sinkhorn-Knopp algorithm} with complexity $O(M^2/\delta)$. Moreover, this can be easily implemented efficiently in \verb|C++| as the algorithm corresponds to an alternating normalization between rows and columns.

\textbf{Quadratic Entropy.} We can use the following quadratic entropy
\begin{equation*}
    H(Q) \defeq -\|Q\|_{F}^2 = -\sum_{i,j=1}^MQ_{ij}^2.
\end{equation*}
The projection has essentially the same complexity as the entropy on marginals described above \cite{blondel2017smooth} and it provides sparse solutions.
The algorithm consists in minimizing an unconstrained smooth and non-strongly convex function. The computation of the gradient requires $M$ euclidean projections to the simplex $\Delta_M$. Each projection can be performed exactly in worst-case $O(M\log M)$ using the algorithm by \cite{michelot1986finite} and in expected $O(M)$ using the randomized pivot algorithm of \cite{duchi2008efficient}. The resulting computational complexity is of $O(M^2/\delta)$. Note that even though the complexity is the same as for the entropic regularization, the implementation is more involved and difficult to speed up.

In our experiments we focus on the entropy on marginals \eqref{eq:entropyonmarginals}. We now compute the constants.

\textbf{Norm $\boldsymbol{\|\cdot\|_{\calM}}$ and constants $\boldsymbol{R_{\calM},\sigma_\calM}$.} If we consider $\|\cdot\|_{\calM} = \|\cdot\|_{1}$, we have that $\sigma_{\calM}=1$ and $R_{\calM}^2 = M$.

\textbf{Computation of the smoothness constants $\boldsymbol{(\beta_{11}, \beta_{12},\beta_{21},\beta_{22})}$.} In this case we obtain $\beta_{11} = \beta_{12} = 0$ and $\beta_{21} = \beta_{22} = 1$. Hence
\begin{equation*}
    L = M.
\end{equation*}

\section{Generalization Bounds for $\boldsymbol{\operatorname{M^4N}}$ solved via GBCFW and Approximate Oracle}

\paragraph{Proof of Theorem 5.2}
Denote by $\widehat{g}_{n,T}$ the result of \cref{alg:gbcfw} where the oracle is approximated via Algorithm~\ref{alg:spmp}. In the same setting of Section~E, by applying Theorem~\ref{thm:gen-bound-approx}, bounding $L, B$ as in the proof of Theorem~\ref{th:generalizationerm} and applying the comparison inequality in Theorem~\ref{th:comparisoninequality}, we have that the following holds with probability $1-\delta$
\begin{equation*}
    \calE(d\circ g_n) - \calE(f^\star) \leq 2(\calR_n^\lambda(\widehat{g}_{n,T}) - \calR_n^\lambda(g_n^\lambda)) + M \|\phi(f^\star)\|_{\calG}\sqrt{\frac{\log(1/\delta)}{n}},
\end{equation*}
when $\lambda$ is chosen as $\lambda = \kappa L\log^{1/2}(1/\delta)n^{-1/2}$ and $M$ defined as in Theorem~\ref{th:generalizationerm}. Denote by $\varepsilon_{\text{opt}} = \calR_n^\lambda(\widehat{g}_{n,T}) - \calR_n^\lambda(g_n^\lambda)$. 
The result is obtained by optimizing until $\varepsilon_{\text{opt}} = O\Big( \|\phi(f^\star)\|_{\calG}\sqrt{\frac{\log(1/\delta)}{n}}\Big)$, we have that 
\begin{equation*}
    \calR(\widehat g_n) - \calR(g^\star) \leq O\Big( \|\phi(f^\star)\|_{\calG}\sqrt{\frac{\log(1/\delta)}{n}}\Big).
\end{equation*}
According to Theorem~5.1 and \ref{th:appspmp} this is obtained with a number of steps for \cref{alg:gbcfw} of $T = O(n)$ and \cref{alg:spmp} in the order of $O(\sqrt{n})$, for a total computational complexity of $O(n\sqrt{n})$.
\qed{}

\end{document}